\documentclass[runningheads,a4paper]{llncs}

\pdfoutput=1

\usepackage{chngpage}
\usepackage[labelfont=bf,textfont=normalfont,singlelinecheck=off,justification=raggedright]{subcaption}

\usepackage{empheq}

\usepackage{paralist}
\usepackage{amssymb}
\usepackage{amsmath}
\usepackage{pifont}
\usepackage{graphicx}
\usepackage{url}
\usepackage{xspace}
\usepackage{hyperref}
\usepackage{xcolor}
\usepackage{booktabs} 
\usepackage{makecell}
\usepackage{paralist, tabularx}
\usepackage{algorithm}
\usepackage{algorithmic}
\usepackage{caption}
\usepackage{subcaption}
\usepackage[T1]{fontenc}
\usepackage{multirow,multicol}
\usepackage{wrapfig}
\usepackage[most]{tcolorbox}
\tcbset{myformula/.style={colback=white, 
		colframe=black, 
		top=0pt,bottom=5pt,left=0pt,right=0pt,
		boxsep=0pt,
		arc=0pt,
		outer arc=0pt,boxrule=0.5pt,
}}

\makeatletter
\g@addto@macro\normalsize{%
	\setlength\abovedisplayskip{2pt}
	\setlength\belowdisplayskip{2pt}
	\setlength\abovedisplayshortskip{2pt}
	\setlength\belowdisplayshortskip{2pt}
}
\makeatother

\usepackage{tikz}
\usetikzlibrary{positioning}

\newcommand{\fgr}[3][\relax]{%
	\begin{figure}[htp]%
		\centering
		\includegraphics[#2]{#3}%
		\ifx\relax#1\else\caption{{#1}}\fi
	\end{figure}%
}

\usepackage{xcolor}

\newcommand{\cbit}{\begin{compactitem}}
	\newcommand{\ceit}{\end{compactitem}}
\newcommand{\cben}{\begin{compactenum}}
	\newcommand{\ceen}{\end{compactenum}}

\newcommand{\bal}{\begin{align}}
\newcommand{\ean}{\end{align}}
\newcommand{\bit}{\begin{itemize}}
\newcommand{\eit}{\end{itemize}}
\newcommand{\ben}{\begin{enumerate}}
\newcommand{\een}{\end{enumerate}}
\newcommand{\beq}{\begin{equation}}
\newcommand{\eeq}{\end{equation}}
\newcommand{\R}{\mathbb{R}}

\newcommand{\mV}{\mathcal{V}}
\newcommand{\mD}{\mathcal{D}}
\newcommand{\mR}{\mathcal{R}}
\newcommand{\mN}{\mathcal{N}}
\newcommand{\mL}{\mathcal{L}}
\newcommand{\mO}{\mathcal{O}}

\newcommand{\mRt}{\mathcal{R}_{\text{test}}}
\newcommand{\mRth}{\hat{\mathcal{R}}_{\text{test}}}
\newcommand{\mRh}{\hat{\mathcal{R}}}

\newcommand{\riskd}{{\sc Risk-Doc}\xspace}
\newcommand{\risks}{{\sc{Risk-Sen}}\xspace}
\newcommand{\nyt}{{\sc{NYT-Dstr}}\xspace}

\makeatletter
\renewcommand*\env@matrix[1][*\c@MaxMatrixCols c]{%
	\hskip -\arraycolsep
	\let\@ifnextchar\new@ifnextchar
	\array{#1}}
\makeatother
\newcommand{\byy}{\mathbf{y}}

\newcommand{\bA}{\mathbf{A}}
\newcommand{\bQk}{\mathbf{R}_k}
\newcommand{\bR}{\mathbf{R}}
\newcommand{\bN}{\mathbf{N}}
\newcommand{\bX}{\mathbf{X}}
\newcommand{\bW}{\mathbf{W}}
\newcommand{\bG}{\mathbf{G}}
\newcommand{\ba}{\mathbf{a}}
\newcommand{\by}{\mathbf{v}_0}
\newcommand{\byk}{\mathbf{v}_k}
\newcommand{\bykp}{\mathbf{v}_{k'}}
\newcommand{\bz}{\mathbf{z}}
\newcommand{\bzk}{\mathbf{z}_k}
\newcommand{\bzz}{\mathbf{z}_0}
\newcommand{\bx}{\mathbf{x}}

\newcommand{\bone}{\mathbf{1}}
\newcommand{\bw}{\mathbf{w}}
\newcommand{\br}{\mathbf{r}}
\newcommand{\bv}{\mathbf{v}}
\newcommand{\bxp}{\mathbf{x}_{[p]}}
\newcommand{\bxq}{\mathbf{x}_{[q]}}

\newcommand{\bD}{\mathbf{D}}

\newcommand{\bzero}{\boldsymbol{0}}

\newcommand{\bphi}{\boldsymbol{\Phi}}

\newcommand{\hide}[1]{}

\tikzset{
	ncbar angle/.initial=90,
	ncbar/.style={
		to path=(\tikztostart)
		-- ($(\tikztostart)!#1!\pgfkeysvalueof{/tikz/ncbar angle}:(\tikztotarget)$)
		-- ($(\tikztotarget)!($(\tikztostart)!#1!\pgfkeysvalueof{/tikz/ncbar angle}:(\tikztotarget)$)!\pgfkeysvalueof{/tikz/ncbar angle}:(\tikztostart)$)
		-- (\tikztotarget)
	},
	ncbar/.default=0.5cm,
}

\tikzset{square left brace/.style={ncbar=0.5cm}}
\tikzset{square right brace/.style={ncbar=-0.5cm}}

\newcommand{\method}{{\sc RaRecognize}\xspace}
\newcommand{\lac}{{\sc L2AC}\xspace}
\newcommand{\senc}{{\sc SENCForest}\xspace}
\newcommand{\bl}{{\sc Baseline}\xspace}

\newcommand{\rarerate}{{\it acc(rare)}\xspace}
\newcommand{\precision}{{\it Precision}\xspace}
\newcommand{\recall}{{\it Recall}\xspace}
\newcommand{\preseen}{{\it Precision (seen)}\xspace}
\newcommand{\fo}{{\it F1}\xspace}
\newcommand{\reseen}{{\it Recall (seen)}\xspace}
\newcommand{\reunseen}{{\it Recall (unseen)}\xspace}
\newcommand{\rulesep}{\unskip\ \vrule\ }

\definecolor{OliveGreen}{rgb}{0,0.6,0}

\makeatletter
\newcommand\footnoteref[1]{\protected@xdef\@thefnmark{\ref{#1}}\@footnotemark}
\makeatother

\newsavebox\CBox
\def\textBF#1{\sbox\CBox{#1}\resizebox{\wd\CBox}{\ht\CBox}{\textbf{#1}}}

\hypersetup{
	colorlinks,
	citecolor={green},
	urlcolor={blue}
}

\title{Continual Rare-Class Recognition \\with Emerging Novel Subclasses}

\author{Hung Nguyen \quad\quad Xuejian Wang \quad\quad Leman Akoglu}
\institute{Carnegie Mellon University \\
 Heinz College of Information Systems and Public Policy \\
   \email{\{hungnguy, xuejianw, lakoglu\}@andrew.cmu.edu}}

\date{}

\begin{document}
\maketitle

\begin{abstract} 

Given a labeled dataset that contains a rare (or minority) class of \textit{of-interest} instances, as well as a large class of instances that are \textit{not} of interest,
how can we learn to recognize future \textit{of-interest} instances over a continuous stream?
We introduce \method, which  ($i$) estimates a {\em general} decision boundary between the rare and the majority class, ($ii$) learns to recognize individual rare subclasses that exist within the training data, as well as ($iii$) flags instances from previously unseen rare subclasses as newly emerging.
The learner in $(i)$ is general in the sense that by construction it is dissimilar to the \textit{specialized} learners in $(ii)$, thus distinguishes minority from the majority without overly tuning to what is seen in the training data.  
Thanks to this generality, \method ignores all future instances that it labels as majority
and recognizes the recurrent as well as emerging {\em rare} subclasses only. This saves effort at 
test time as well as ensures that the model size grows moderately over time as it only maintains specialized minority learners.
Through extensive experiments, we show that \method outperforms state-of-the art baselines on three real-world datasets that contain corporate-risk and disaster documents as rare classes.

\end{abstract}

\section{Introduction}
\label{sec:intro}

Given a labeled dataset containing (1) a rare (or minority) class of \textit{of-interest} documents, and (2) a large set of \textit{not-of-interest} documents,
how can we learn a model that can effectively identify future \textit{of-interest} documents over a continuous stream?
Different from the traditional classification setup, the stream might contain \textit{of-interest} (as well as \textit{not-of-interest}) documents from \textit{novel subclasses that were not seen in the training data}. Therefore, the model is required to continually recognize both the recurring as well as the emerging instances from the underlying rare class distribution.
 
 
 Let us motivate this setting with a couple of real-world examples.
 Suppose we are given a large collection of social media documents (e.g. Twitter posts).
 A subset of the collection is labeled as \textit{risky}, indicating posts that constitute (financial, reputational, etc.) risk to a corporation. The rest (majority) of the collection is \textit{not-risky}. The goal is then to learn a model that can continually identify future posts that are {\em risky} over the social-media stream.
 Here, the rare class contains \textit{risky} documents of a few known types, 
  such as bankruptcy, corruption, and spying.
  However, it is unrealistic to assume that it contains examples from all possible risk types---given the large spectrum, labeling effort, and potentially evolving nature of risk.

Consider another case where the training set consists of news articles. A subset of the articles belongs to the rare class of {\em disasters}, indicating news about natural or man-made disasters. The rest are \textit{not-disaster} articles.
 Similar to the first case, the rare class might contain articles about floods, earthquakes, etc. however it is hard to imagine it would contain instances from all possible types of disasters. The goal is to learn to  continually recognize future articles on disasters.


In both examples above, the model needs to learn from and generalize beyond the labeled data so as to recognize future rare-class instances, both from \textit{recurring} (i.e., seen in the training data) as well as from \textit{novel subclasses}; for instance sexual assault, cyber attack, etc. 
in risk domain and explosions, landslides, etc. in disasters domain. 
In machine learning terms, this is a very challenging setup in which the learner needs to generalize not only to unseen instances but also to \textit{unseen} \textit{distributions}. In other words, this setting involves test data that has a related yet different distribution than the data the model was trained on.


The stream classification problem under emerging novel classes has been studied by both machine learning and data mining communities.
The area is referred to under various names including 
open-world classification \cite{shu2017doc,shu2018unseen},
life-long learning \cite{chen2016lifelong}, and
continual learning \cite{shin2017continual}.
In principle, these build a ``never-ending learner'' that can (1) assign those recurring instances from known old classes to their respective class, (2) recognize emerging classes, and (3) grow/extend the current learner to incorporate the new class(es).
The existing methods differ in terms of accuracy-efficiency trade-offs 
and various assumptions that they make. 
 (See Section \ref{sec:relatedwork} for detailed related work.)
 A common challenge that all of them face is what is known as {\em catastrophic forgetting}, mainly due to model growth. In a nutshell, the issue is the challenge of maintaining performance on old classes as the model is constantly grown to accommodate the new ones.
 
\begin{figure}[!h]
  \centering
  \includegraphics[width=1.0\columnwidth]{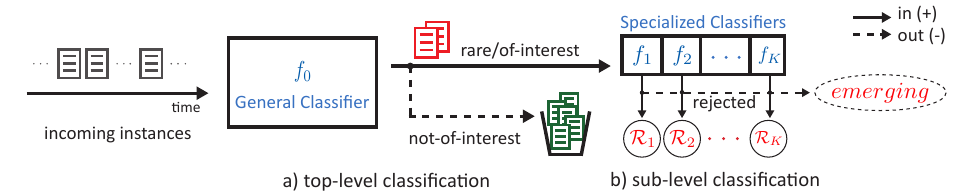}
  \caption{An illustration of the recognition flow in our proposed model.}
  \label{fig:key_idea}
\end{figure}

 Our work is different from all prior work in one key aspect: our goal is not to recognize \textit{any and every} newly emerging class---but only those (sub)classes related to the rare class \textit{of-interest}. That is, our primary goal is to recognize rare-class instances. \textit{Not-of-interest} instances, as long as they are filtered out accurately, are ignored---{\em no matter they are recurrent or novel}, as depicted in Fig.~\ref{fig:key_idea}.
 This way, we carefully avoid the aforementioned issue that current models face. Our model grows slowly, only when novel \textit{rare} subclasses are recognized. Thanks to a moderate model size (by definition, rare subclasses are far fewer), our model is not only less prone to catastrophic forgetting but also (a) is faster at test time, and (b) requires much less memory.
 
 \hide{
 To this end, we first ($i$) build \textit{specialized} learners that recognize individual rare subclasses that exist within the training data. Then, we ($ii$) estimate a {\em general} decision boundary between the minority- (i.e., rare) and majority-class instances.
 This binary classifier is general in the sense that it is sufficiently different from the \textit{specialized} learners by construction, and thus distinguishes minority from the majority without overly tuning to what is seen in the training data.
 For newly-arriving test instances over the stream, we simply ignore those labeled as majority-class by the {\em general} learner. Otherwise the instance belongs to the rare class, in which case the specialized learners are employed to either assign to existing rare subclasses or abstain/reject the instance as newly emerging.
}
 
 We summarize the main contributions of this work as follows.
 \bit
 \item {\bf Problem and Formulation:~} We address the problem of recognizing instances from a rare, {\em of-interest} class over a stream continually. The setting differs from traditional (binary) classification in that the data distribution (for both rare and majority class) might change over time, where novel subclasses emerge.
 We formulate a new model called \method that {\em simultaneously} learns ($i$) a separate specialized classifier (SC) that recognizes an individual rare subclass, as well as ($ii$) a general classifier (GC) that separates rare instances from the majority. While being discriminative, GC is constructed to be dissimilar to the individual SCs such that it can generalize without overly tuning to seen rare subclasses in the training data.

 \item {\bf Efficient Algorithm:~} Our proposed solution exhibits two key properties: runtime and memory efficiency; both essential for the 
 stream setting.
 Given a new instance 
 that GC labels as belonging to the majority class, we simply do nothing---no matter it is recurrent or emerging.
 By not processing the majority of the incoming instances, we achieve \textit{fast response time}.
 Moreover \method remains compact, i.e. \textit{memory-efficient}, as it requires space linear in the number of \textit{rare} subclasses which only grows slowly. 
 
 \item {\bf Applications:~} Recognizing recurrent as well as novel instances that belong to a certain class \textit{of-interest} is a broad problem that finds numerous applications, e.g. in monitoring and surveillance.
 For example, such instances could be production-line items with the goal to continually recognize faulty ones where novel fault types might emerge over time. They could also be public documents, such as social media posts, where the goal is to recognize public posts {\em of-interest} such as bullying, shaming, disasters, threat, etc. 
 \eit
 \noindent
 \textbf{Reproducibility:~} We share the source code for \method and
our public-domain datasets at \url{{https://github.com/hungnt55/RaRecognize}}.

\section{Problem Setup and Preliminary Data Analysis}
\label{sec:motivation}

{\bf Problem Setup and Overview.~}
We start by introducing the problem statement more formally with proper notation.
As input, a labeled training dataset  $\mD = \mR\cup \mN \in \R^{n\times d}$ containing $n$ $d$-dimensional instances is provided. The set $\mR = \{(\bx_1,y_1), (\bx_2,y_2), \ldots, (\bx_{n_0},y_{n_0})\}$  consists of $|\mR|=n_0$ instances belonging to the {\em of-interest} rare class where $y_i=+1$ for $i=1,\ldots,n_0$ and the set
$\mN = \{(\bx_{n_0+1},y_{n_0+1}), \ldots, (\bx_{n},y_{n})\}$  consists of $|\mN|=(n-n_0)$ instances from the {\em not-of-interest} class where $y_i=-1$ for $i=(n_0+1),\ldots, n$.
 Without loss of generality, we will refer to the data instances as documents and to the rare class as the \textit{risk} class in the rest of this section to present our ideas more concretely.

Given $\mD$, the goal is to recognize future \textit{risk} documents, \textit{either recurring or newly emerging}, over a stream (or set) of new documents $\bx_{n+1}, \bx_{n+2}, \ldots$ (here, each document has a vector representation denoted by $\bx$ such as bag-of-words, embedding, etc.). The new documents may associate with recurring risk, i.e., belong to known/seen risk subclasses $1, \ldots, K$ in $\mR$. They may also be emerging, i.e., from previously unknown/unseen new risk subclasses ${(K+1)}, {(K+2)}, \ldots$; which differentiates our setup from the traditional classification problem.

Therefore, we start by decomposing $\mR$ into known risk subclasses, $\mR = \bigcup_{k=1}^K \mR_k$, where 
$\mR_k$ 
contains the documents that belong to the $k$th risk subclass.
Given $\{\{\mR_1,\ldots, \mR_K\}, \mN\}$ our approach involves simultaneously training the following two types of classifiers:

\cben
\item A \textit{general} classifier (GC) $f_0$ to separate $\mR$ and $\mN$ that can generalize to unseen subclasses of $\mR$,
\item A \textit{specialized} classifier (SC) $f_k$, $k=1\ldots K$, to separate $\mR_k$ and $\mR\backslash \mR_k$. 
\ceen

At test time, we first employ $f_0$.
Our goal is not to recognize every emerging novel class, but only the novel risk subclasses (in addition to recurring ones), thus our first step is to recognize risk. If $f_0$ labels an incoming document $\bx$ as $-1$ (i.e., not-risk), we discard it.
Otherwise, the incoming document is flagged as risky.
For only those labeled as $+1$,  we employ 
 $f_k$'s to further identify the type of risk.
 Among the $f_k$'s that accept $\bx$ as belonging to the $k$th risk subclass, 
 we assign it to the subclass that is $\arg\max_k f_k(\bx)$.
 If all $f_k$'s reject, then $\bx$ is considered to be associated with a new type of emerging risk. (See Fig. \ref{fig:key_idea}.)

\noindent
{\bf The classifier models.~}
Our risk detector is $f_0$ which we learn using the entire labeled dataset $\mD$.
As such, it is trained on a few known risk subclasses in $\mR$ but is desired to be \textit{general} enough to recognize other types of future risk.

To achieve this generality, our main idea is to avoid building $f_0$ on factors that are too specific to any known risk subclass (such that $f_0$ is not overly fit to existing or known risk types)
but rather, to identify broad factors about risk that are \textit{common} to all risk subclasses (such that $f_0$ can employ this broader view to spot risk at large).

In fact, factors specific to the known risk subclasses are to be captured by the corresponding $f_k$'s. Then, $f_0$ is to identify discriminative signals of risk that are sufficiently different from those used by all $f_k$'s.
Moreover, each $f_k$ should differ from other $f_{k'}$'s, $k'\neq k$, to ensure that they are as {\em specialized} as possible to their respective risk types.  
Such dependence among the models is exactly why we train all these ($K+1$) classifiers 
\textit{simultaneously}, to enforce the aforementioned constraints conjointly.
We present our specific model formulation and optimization in Section \ref{sec:proposed}.

\noindent
{\bf Preliminary Data Analysis.~}
Before model formulation, we perform an exploratory analysis on 
one of our real-world datasets containing documents labeled as risky and not-risky.
The goal of the analysis is to see if our hypothesized ideas get realized in the data.

In particular, 
we aim to find out if there exists (1) factors that are specific to each risk subclass,
as well as (2) factors beyond those specific ones that are still discriminative of risk.
For simplicity and interpretability, we use the bag-of-words representation of the data in this section, thus factors correspond to individual words. However, our proposed model can handle other document vector representations in general.  

To this end, we formulate a constrained optimization problem to find 
word sets that cover or characterize different document sets.
Here, we define a word to \textit{cover} a document if the word appears in it at least once. 
Given the set of unique words $\mV$, $|\mV|=d$, we look for
a set of words $\mV_k \subset \mV$ that covers all the documents in $\mR_k$ but as few as those in $\mR \backslash \mR_k$ for all $k=1\ldots K$ (i.e., specific words for each risk subclass), and another set of words $\mV_0 \subset \mV$ that covers all risk documents in $\mR$ but as few as those in $\mN$.
We restrict the word sets to be \textit{non-overlapping}, i.e., $\mV_k \cap \mV_{k'} = \emptyset\;\; \forall k,k' \in \{0,\ldots,K\}$, such that each word can only characterize either one of $\mR_1, \ldots, \mR_K$ or $\mR$ at large.
Under these conditions, if we could find a set $\mV_0$ that shares no words with any $\mV_k$'s while still being able to cover the risky documents but only a few (if at all) not-risky ones, then we can conclude that broad risk terms exist and a \textit{general} $f_0$ can be trained.

Our setup is a constrained mixed-integer linear program (MILP) as follows:
{\footnotesize{
\begin{align}
\label{opt}
\min_{\bphi} \;\;\; & |\by| + \sum_{k=1}^{K} |\byk| +  o + \alpha + \beta & \nonumber \\
\text{s.t.} \;\;\;  & \bzk + \bQk \cdot \byk \geq \bone \;\;\; \forall k=1\ldots K  
&\triangleright \text{ risk subclass coverage with exoneration} \nonumber\\
& \bzz + \sum_{k=1}^{K}\bQk \cdot \by \geq \bone \;\;\;\; & \triangleright \text{ risk coverage with exoneration}\nonumber\\
 & |\bzz| + \sum_{k=1}^K |\bzk| \leq o \;\;\;\; & \triangleright \text{ \# unexplained documents less than $o$}\nonumber\\
& \by + \sum_{k=1}^K \byk \leq \bone \;\;\;\; & \triangleright \text{ each word used for at most 1 set}\nonumber\\
 & \sum_{i\in \mN} \bN_{i} \cdot {\by} \leq \alpha  \;\;\;\; &\triangleright \text{ cap on cross-coverage of not-risk} \nonumber\\
 & \sum_{k=1}^K \sum_{i\in \bQk} \sum_{k'\neq k}  \bR_{k,i} \cdot {\bykp} \leq \beta  \;\;\;\; &\triangleright \text{ cap on cross-coverage among risk subclasses} 
 \nonumber
\end{align}
}}

The program is parameterized by $\bphi= \{\{\byk\}_{k=0}^K, \{\bzk\}_{k=0}^K, o, \alpha, \beta\}$. 
$\bQk \in \R^{n_k\times d}$ denotes the data matrix encoding the word occurrences for $n_k$ documents in risk subclass $k=1\ldots K$, and $\bN \in \R^{(n-n_0)\times d}$ is the corresponding data matrix for the not-risk documents.
 $\by \in \R^d$ and $\byk \in \R^d$'s depict (binary) variables to be estimated that capture the word assignments to the sets $\mV_0$ and $\mV_k$'s respectively. (e.g., $j$th entry of $\by$ is set to 1 if word $j$ is assigned to $\mV_0$ and to 0 otherwise.)

The first set of constraints are coverage constraints for risk subclasses: each document in $\mR_k$ should contain at least one of the words in their assigned set.
Enforcing this constraint for \textit{all} the documents is too strict, hence we introduce additional (binary) variables $\bzk\in \R^{n_k}$'s that ``exonerate'' some documents. When $i$th entry of $\bzk$ is set to 1, then document $i$ in $\mR_k$ is allowed to have no matching words, as  $\bzk=1$ ensures the constraint holds even without any match.
The second constraint is similar, and enforces coverage for the combined set of risk documents, also with exoneration. Of course, we aim to cover as many documents as possible and thus upper-bound the total number of exonerated documents by $o$, where $|\bz|$ denotes the total number of 1s in vector $\bz$.
Next is the no-overlaps constraint, enforcing each word is assigned to only one set. The final two constraints are cross-coverage constraints; the former ensures that the words assigned to $\mV_0$ have less than $\alpha$ number of matches in not-risk documents and the latter ensures that the words assigned to each risk subclass have less than $\beta$ number of matches outside the respective document set in total.

Ideally all of $o$, $\alpha$, and $\beta$ are zero; that is, all documents are covered without any exoneration and no cross-coverage exists.
However, 
that yields no feasible solution. Instead, 
we define them as scalar upper-bound variables added to our minimization objective toward setting them to as small values as possible. 
Finally, our objective aims to find the smallest-size possible word sets. This  ensures that the most important words are selected which also facilitates interpretability. 

\begin{figure}[!t]
  \centering
  \includegraphics[width=1.0\columnwidth]{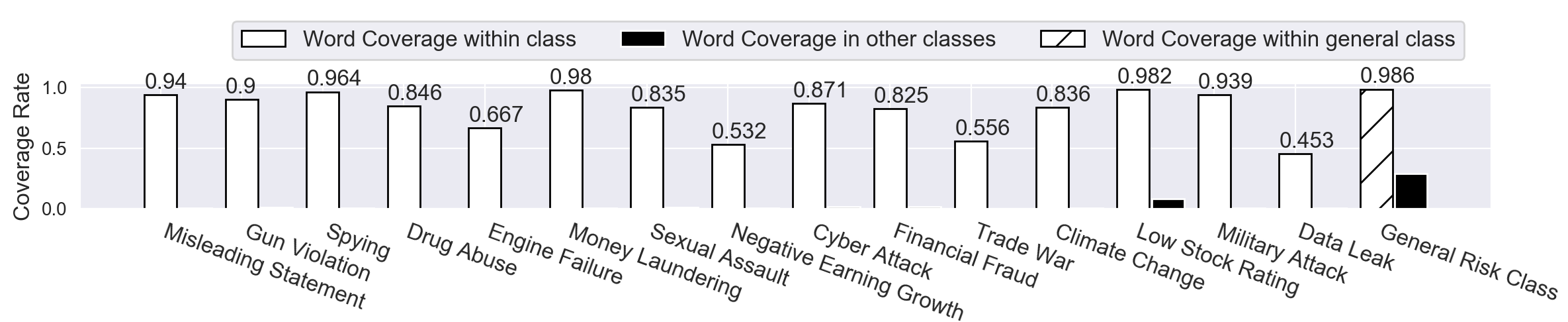}
  \caption{Within- and cross-coverage rates of $\mV_k$'s and $\mV_0$ (resp.) for $\mR_k$'s and $\mR$.}
  \label{fig:cover_ip}
\end{figure}
\begin{figure}[!t] 
	\centering
	\begin{tabular}{cp{0.1cm}cp{0.1cm}cp{0.1cm}|p{0.1cm}c}
		\includegraphics[width=0.2\linewidth]{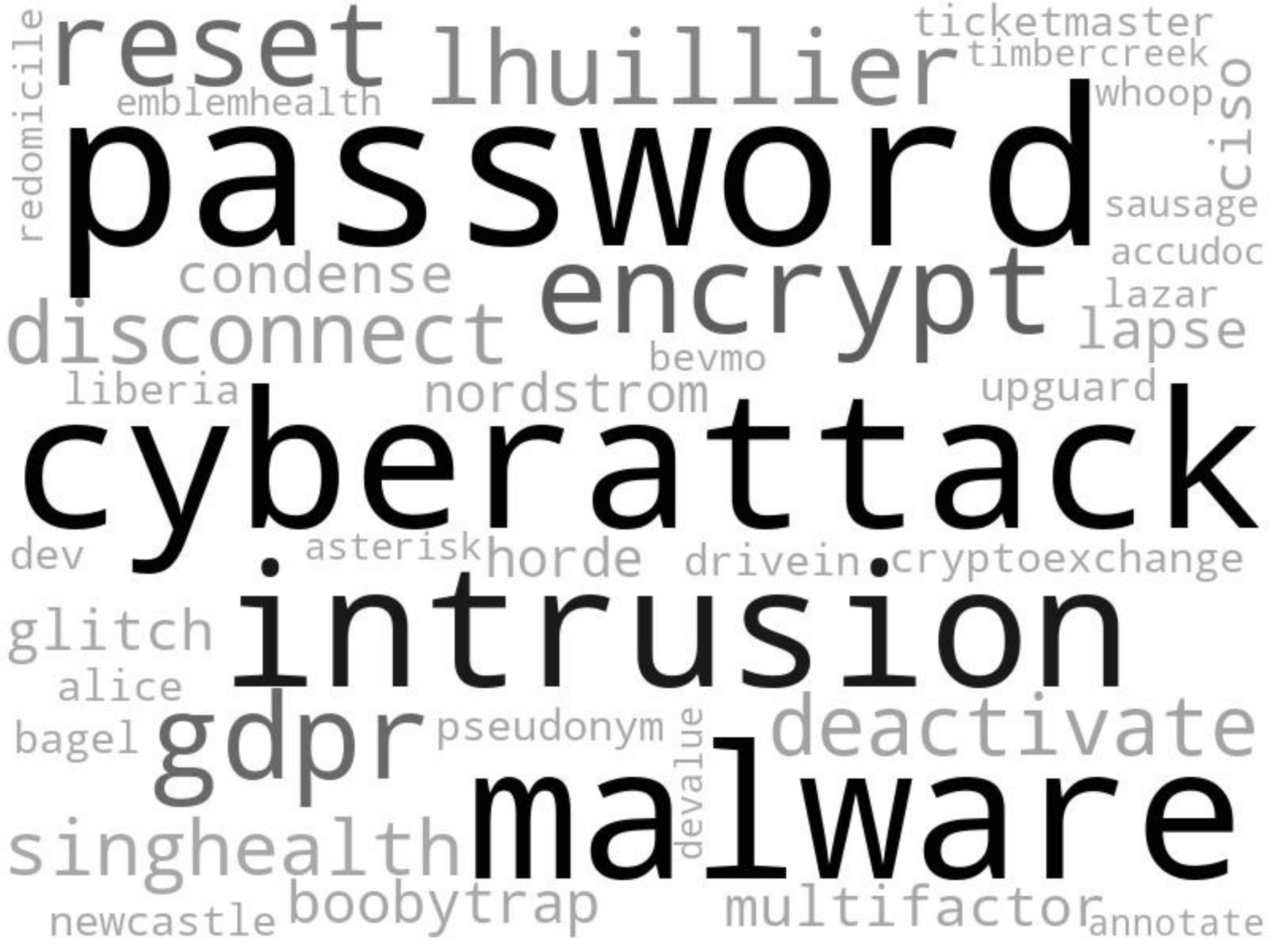} & &
		\includegraphics[width=0.2\linewidth]{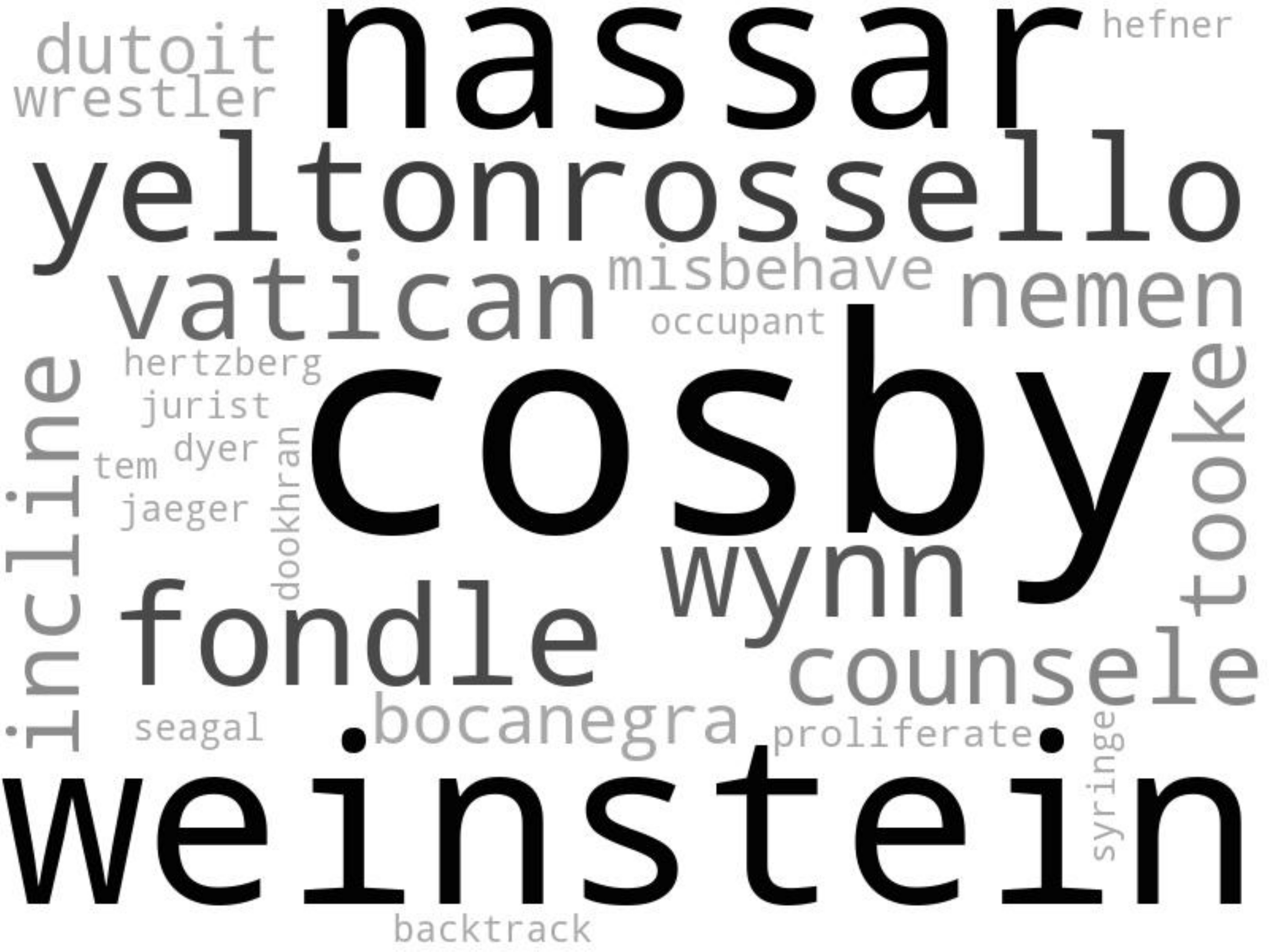} & & 
		\includegraphics[width=0.2\linewidth]{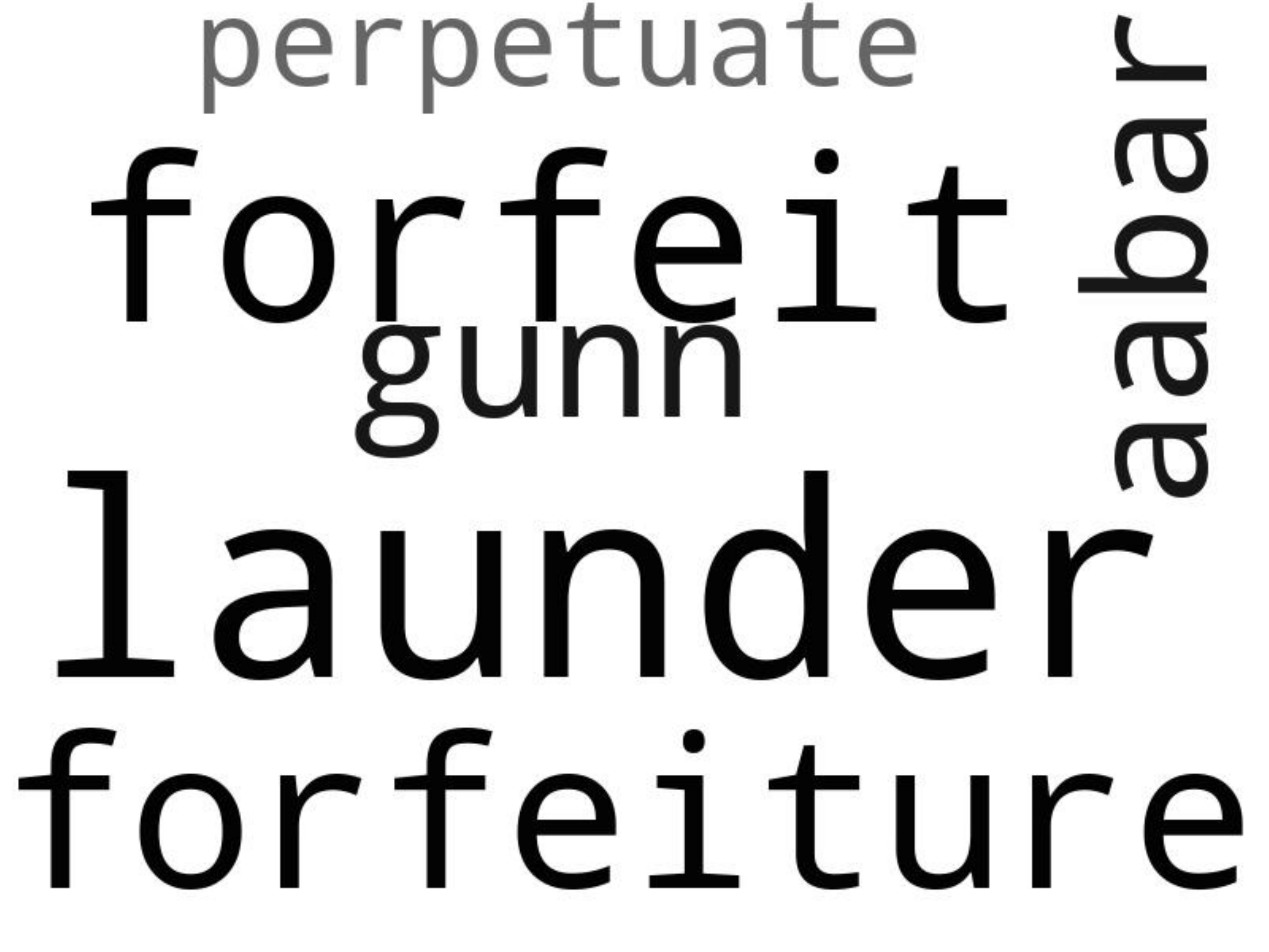}  & & &
		\includegraphics[width=0.2\linewidth]{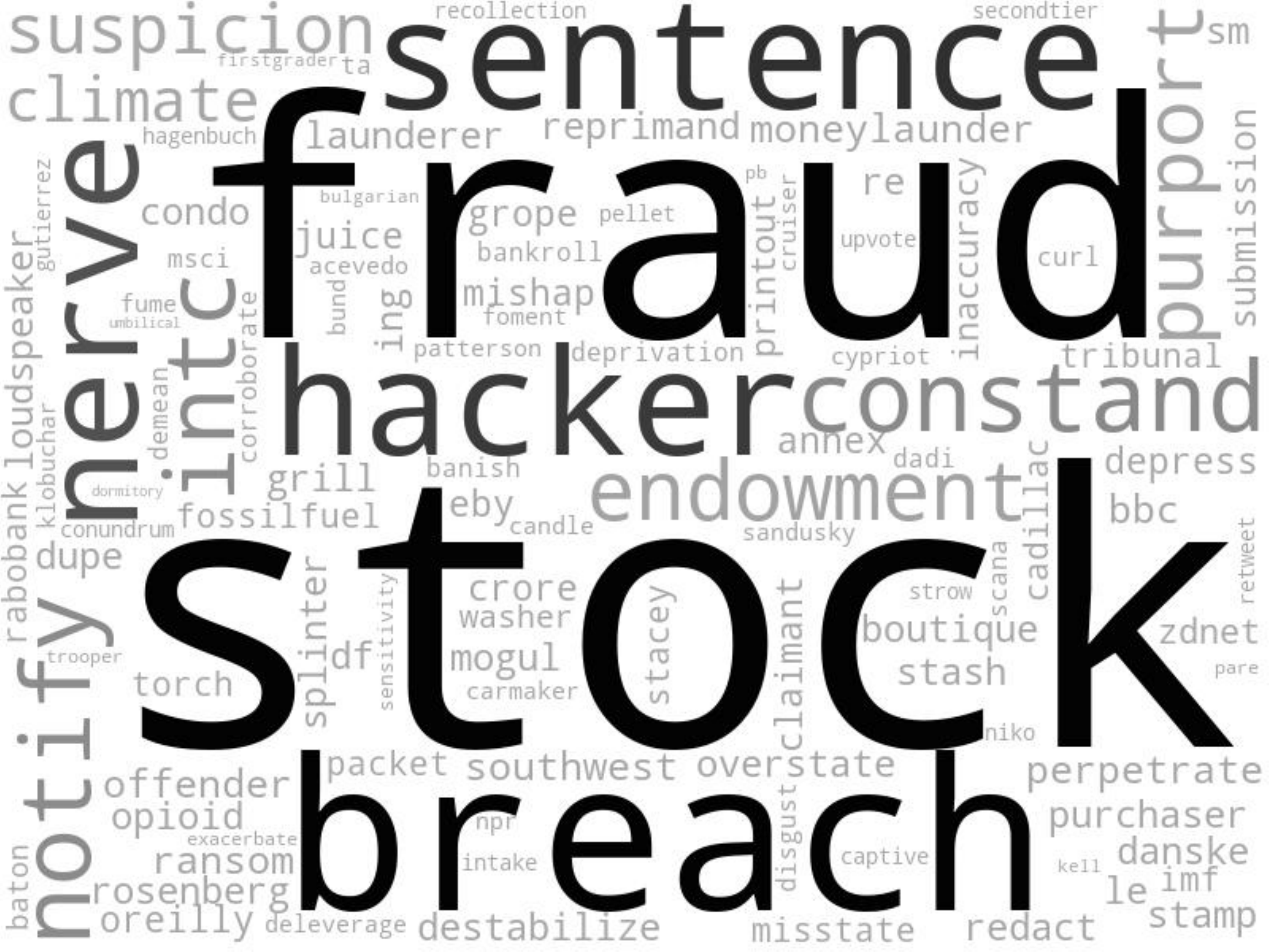}  \\
		($\mV_1$) Cyber attack & & ($\mV_2$) Sexual assault & &
		($\mV_3$) Money laundering & & & ($\mV_0$) General risk
	\end{tabular}
	\caption{Wordclouds representing 3 example risk subclasses and the overall risk class. Notice that the former are quite specific, and the latter are broader.}
	\label{fig:word_cloud_IP}
\end{figure}

We provide an exploratory analysis on a dataset containing corporate risk documents as the {\em of-interest} class.
It contains 15 risk subclasses as outlined by Fig.~\ref{fig:cover_ip}. (See Sec. \ref{sec:setup} for details.)
First the quantitative measures:
as shown in the figure, the MILP finds word sets $\mV_k$'s with at least {82.5\%} up to 98.2\% coverage for {11/15} of the subclasses with an overall coverage of 96.7\% (rest are the exonerated ones). Moreover, cross-coverage is either zero or very low for all the subclasses. These suggest that accurate SCs can be learned.
Importantly, there exist words $\mV_0$ that are \textit{distinct} from all $\mV_k$'s and yet able to cover {98.6\%} of the overall risk documents, promising that a broad GC can be learned.  

To equip the reader with intuition, we present the selected words for 3 example risk subclasses
along with the general risk class words in Fig.~\ref{fig:word_cloud_IP} (word size is proportional to the within vs. cross-coverage ratio). It is easy to see that very specific words are selected for subclasses; such as 
 \texttt{password}, \texttt{cyberattack}, \texttt{malware} for Cyber attack, and \texttt{cosby}, \texttt{weinstein}, \texttt{fondle} for Sexual assault.
On the other hand, in the general risk class, a set of broader corporate risk words appear
such as \texttt{fraud}, \texttt{stock}, \texttt{breach} and \texttt{sentence}.

These preliminary results show promise for the feasibility of our hypothesized models and
demonstrate the rationale behind our proposed \method, which we formally introduce next.

\section{Continual Rare-Class Recognition}
\label{sec:proposed}


In this section, we introduce the individual components of our model,
present the underlying reasoning for our formulation, show convexity and present the optimization steps, and conclude with space and time-complexity analysis.

\subsection{Model Formulation}

As discussed in the previous section, our goal is to learn (1) specialized classifiers $f_k$'s and (2) a general classifier $f_0$.

The {\bf specialized classifier} $f_k$, $k=1\ldots K$, is to learn a decision boundary that separates the $k$th rare subclass instances $\mR_k$ from the remainder of rare instances $\mR\backslash \mR_k$.
Let us write down the regularized loss function for each $f_k$ as
\begin{align}
\mL(f_k; \bw_k, b_k) & = 
\sum_{i=1}^{n_0} 
\underbrace{\max\big(0, \big[1-y_i (\bw_k^T \bx_i + b_k)\big]\big)}_{\ell(\bx_i,y_i; \bw_k,b_k)}
+ \frac{\lambda_k}{2} \|\bw_k\|^2 
\end{align}
where $y_i=+1$ for $\bx_i\in \mR_k$ and $y_i=-1$ otherwise.
We adopt the hinge loss and the ridge regularization as in Eq. (1), 
however, one could instead use other loss functions, such as the logistic, exponential or cross-entropy losses, 
as well as other norms for regularization.

The {\bf general classifier} $f_0$ is to separate rare class instances $\mR$ from the majority instances $\mN$, without relying on factors specific to known rare subclasses.
One way to achieve this \textit{de-correlation} is to enforce $f_0$ to learn coefficients $\bw_0$ that are different from all $\bw_k$'s. The loss function can be written as
\beq
\label{fzero}
\mL(f_0; \bw_0, b_0)  = \sum_{i=1}^{n} \ell(\bx_i,y_i; \bw_0,b_0) + \frac{\lambda_0}{2} \|\bw_0\|^2 + \frac{\mu}{2} \sum_{k=1}^{K} \bw_0^T \bw_k  \;,
\eeq
where $y_i=+1$ for $\bx_i\in \mR = \{\mR_1 \cup \ldots \cup \mR_K\}$ and $y_i=-1$ otherwise.
As required, the third term in Eq. \eqref{fzero} penalizes $\bw_0$ being correlated with any $\bw_k$, enforcing it to be as orthogonal to $\bw_k$'s as possible.
However, it does not prevent $\bw_0$ from capturing \textit{different yet correlated features} to those captured by $\bw_k$'s. This issue can arise when features exhibit multi-collinearity.

For example, in a document dataset
the words {\tt earthquake}, {\tt shockwave}, and {\tt aftershock} could be collinear. In this case it is possible that $f_k$ estimates large coefficients on  a strict subset of these words (e.g., {\tt shockwave} and {\tt aftershock}) as they are redundant. This leaves room for $f_0$ to capitalize on the remaining words (e.g., {\tt earthquake}), which is undesirable since we aim $f_0$ to learn about the rare class boundaries beyond the specifics of the known subclasses. 

Therefore, we reformulate the model correlation penalty as follows.
\beq
\label{fzerore}
\frac{\mu}{2} \sum_{k=1}^{K} \sum_{p,q} \big( w_{0,p} w_{k,q} \;\bx_{[p]}^T \bx_{[q]} \big)^2 = \frac{\mu}{2} \big\| (\bX^T\bX) \odot (\bw_0 \bw_k^T) \big\|_F^2 \;,
\eeq
where $w_{0,p}$ and $w_{k,q}$ denote the $p$th and $q$th entries of $\bw_0$ and $\bw_k$ respectively, $\bX \in \R^{n\times d}$ denotes the input data matrix, $\bxp, \bxq$ respectively denote the $p$th and $q$th columns of $\bX$, and $\odot$ depicts the element-wise multiplication.

We call Eq. \eqref{fzerore} the cross-correlation penalty.
Similarly, we introduce self-correlation penalty to each model $k=0,\ldots, K$ by adding to the respective loss the term
$\sum_{p,q} \big( w_{k,p} w_{k,q} \;\bx_{[p]}^T \bx_{[q]} \big)^2$.
Self-correlation prevents each model from estimating large coefficients on higly correlated (near-redundant) features, which improves sparsity and interpretability, and as we show also ensures convexity.

Then, the overall loss function incorporating the cross- and self-correlation penalty terms for all models $f_0, f_1, \ldots, f_K$ is given as follows.
\begin{tcolorbox}[ams align,myformula]
\mL & = \sum_{i=1}^{n} \ell(\bx_i,y_i; \bw_0,b_0) + \frac{\lambda_0}{2} \|\bw_0\|^2 + \sum_{k=1}^K 
\left[\sum_{i=1}^{n_0} \ell(\bx_i,y_i; \bw_k,b_k) + \frac{\lambda_k}{2} \|\bw_k\|^2 \right]  
\nonumber\\
&  
+  \frac{\mu}{2} \sum_{p,q} 
\bigg\{ 
\underbrace{\frac{1}{2}
	(w^2_{0,p} w^2_{0,q})
	+ \frac{1}{2} \sum_{k=1}^{K} (w^2_{k,p} w^2_{k,q})}_{\text{self-correlation}}
+ \underbrace{w^2_{0,p} (\sum_{k=1}^{K} w^2_{k,q})}_{\text{cross-correlation}}
\bigg\} 
\big(  \bx_{[p]}^T \bx_{[q]} \big)^2 \vspace{0.25in}
\label{eq:loss}
\end{tcolorbox}


\subsection{Convexity and Optimization}

We train all the models $f_0, f_1, \ldots, f_K$ \textit{simultaneously} by minimizing the total overall loss $\mL$. A conjoint optimization is performed because the cross-correlation penalty terms between $\bw_0$ and $\bw_k$'s induce dependence between the models. 

For optimization we employ the accelerated subgradient descent algorithm which is guaranteed to find the global optimum solution because, as we show next, our loss function $\mL$ is convex.

\begin{theorem}
 \textit{The joint loss function $\mL$ involving the cross- and self-correlation penalty terms among $\bw_0, \bw_1, \ldots, \bw_K$ remains convex.}
\end{theorem}
\begin{proof}
	The non-negative sum of convex functions is also convex.
The first line of $\mL$ as given above is known to be convex since $\ell(\cdot)$ (hinge loss) and L-$p$ norms for $p\geq 1$ are both convex. The
proof is then by showing that the overall correlation penalty term in the second line of $\mL$ is also convex by showing that its Hessian matrix is positive semi-definite (PSD).  See Supplementary A.1 for details. \qed
\end{proof}

Since our total loss is a convex function, we can use gradient-based optimization to solve it to optimality.
To this end, we provide the gradient updates for both $\bw_0$ and $\bw_k$'s in closed form as follows.

{\bf Partial derivative of $\mL$ w.r.t. $\bw_0$:~}
\beq
\frac{\partial \mL}{\partial w_{0,p}} = \sum_{i=1}^{n} \frac{\partial \big[1-y_i (\bw_0^T \bx_i + b_0)\big]_+}{\partial w_{0,p}} + \lambda_0 w_{0,p}  + \mu w_{0,p}\; \sum_{q=1}^{d} 
\big( \sum_{k=1}^K w^2_{k,q} + w^2_{0,q} \big) \big(  \bx_{[p]}^T \bx_{[q]} \big)^2 \nonumber
\eeq
where
\beq
\frac{\partial \big[1-y_i (\bw_0^T \bx_i + b_0)\big]_+}{\partial w_{0,p}} =
\begin{cases}
	0 & \;\;\;\; \text{if } y_i (\bw_0^T \bx_i + b_0) \geq 1 \\
	-y_i x_{i,p} & \;\;\;\; \text{otherwise.}
\end{cases}
\eeq

The vector-update $\frac{\partial \mL}{\partial \bw_{0}}$ can be given in matrix-vector form using the above gradients as
\beq
\label{wzero}
\frac{\partial \mL}{\partial \bw_{0}} = 
\big[ 
\bX^T (-\byy \odot \mathbb{I}(1-\byy \odot (\bX \bw_0 + b_0)>0)) 
\;+\; \bw_0 \odot \big( \lambda_0 \bone + \mu (\bX^T\bX)^2 (\sum_k \bw^2_k + \bw^2_0) \big)
\big]
\eeq
where $\mathbb{I}(\cdot)$ is the indicator function, $\bone$ is a length-$n$ all-ones vector, and $(\bA)^2 = \bA \odot \bA$, i.e., element-wise product, for both matrix $\bA$ as well as for vector $\ba$.

{\bf Partial derivative of $\mL$ w.r.t. $\bw_k$:~}
The steps for each $\bw_k$ is similar, we directly provide the vector-update below.
\beq
\label{wk}
\frac{\partial \mL}{\partial \bw_{k}} = 
\big[ 
\bR^T (-\byy_0 \odot \mathbb{I}(1-\byy_0 \odot (\bR \bw_k + b_k)>0)) 
\;+\; \bw_k \odot \big( \lambda_k \bone + \mu (\bX^T\bX)^2 (\bw^2_k + \bw^2_0) \big)
\big]
\eeq
where $\bR \in \R^{n_0\times d}$ and $\byy_0 \in \R^{n_0}$ consist of only the rare-class instances.

\subsection{Time and Space-Complexity Analysis}

{\bf Time Complexity.~} The (first) gradient term in Eq \eqref{wzero} that is related to the hinge-loss is $\mO(nd)$. 
 The (second) term related to the correlation-based regularization requires $\bX^T\bX$ which can be computed in $\mO(nd^2)$ apriori and \textit{reused} over the gradient iterations.
The term $(\sum_k \bw^2_k + \bw^2_0)$ takes $\mO(Kd)$, and its following multiplication with $(\bX^T\bX)^2$ takes an additional $\mO(d^2)$. The remaning operations (summation with $\lambda_0 \bone$ and element-wise product with $\bw_0$) are only $\mO{(d)}$. As such, the overall computational complexity for subgradient descent for $\bw_0$ is 
$\mO(nd^2 + t[nd+Kd+d^2])$, where $t$ is the number of gradient itearations.

The time complexity for updating the $\bw_k$'s can be derived similarly as   
Eq. \eqref{wk} consists of similar terms, which can be written as
$\mO(t[n_0d+d^2])$.
Note that we omit the $\mO(nd^2)$ this time as $\bX^T\bX$ needs to be computed only once and can be shared across all update rules. Moreover, the number of iterations $t$ is the same as before since the parameter estimation is conjoint.
 
 \noindent
{\bf Space Complexity.~} We require $\mO(d^2)$ storage for keeping $(\bX^T\bX)^2$,
$\mO(Kd)$ for all the parameter vectors $\bw_0,\ldots,\bw_K$, and $\mO(nd)$ for storing $\bX$, for an overall
$\mO(d^2 + Kd + nd)$ space complexity.

\noindent
{\bf Remarks on massive and/or high-dimensional datasets:} Note that both time and space complexity of our \method is quadratic in $d$ and linear in $n$.
 We conclude with parting remarks on cases with large $d$ and huge $n$. 
 
First, high-dimensional data with large $d$: In this case, we propose two possible directions to make the problem tractable. Of course, the first one is dimensionality reduction or representation learning.
When the data lies on a relatively low-d manifold, one could instead use compound features.
We apply our \method to document datasets, where compound features are not only fewer but also sufficiently expressive of the data.
The second direction is to get rid of feature correlations, for instance via factor analysis.
This would drop the term $\big(  \bx_{[p]}^T \bx_{[q]} \big)^2$ from $\mL$ (See \eqref{eq:loss}) and lead to updates that are only {\em linear} in $d$.

Next, massive data with huge $n$: We presented our optimization using batch subgradient descent.
When $n$ is very, very large then storing the original data in memory may not be feasible.
We remark that one could directly employ mini-batch or even stochastic gradient descent in such cases, dropping the space requirement to $O(d^2+Kd)$.

\section{Evaluation}
\label{sec:experiments}

We design experiments to evaluate our proposed method with respect to the following questions:
\begin{itemize}
	\item \textbf{RQ1) Top-level classification (via GC $f_0$)}: How does \method perform in differentiating rare-class instances from the majority compared with the state-of-the-art?
	\item \textbf{RQ2) Sub-level classification (via SCs $f_k$'s)}: How does \method perform in recognizing recurrent and emerging rare subclasses among the compared methods? 
	\item \textbf{RQ3) Interpretability}: Can we interpret \method as a model as to what it has learned and what insights can we draw?
	\item \textbf{RQ4) Efficiency}: What is the scalability of \method? How does it compare to the baselines w.r.t. the running time-vs-performance trade-off?
\end{itemize}

\subsection{Experiment Setup}
\label{sec:setup}
\subsubsection{Dataset Description.}
\begin{wraptable}{r}{6.2cm}
	\centering
	\vspace{-0.3in}
	\caption{Summary of datasets.}
	\label{table:data_sum}
	\begin{tabular}{l|r|r|>{\raggedleft} p{0.75cm}|r}
		\toprule
		\textbf{Name} & $|\mR|$ & $|\mN|$ & $K$ & \textbf{Avg. $|\mR_k|$} \\
		\midrule
		\riskd & 2948 & 2777 & 15 & 196.5 \\
		\risks & 1551 & 7755 & 8  & 193.9 \\
		\nyt & 2127 & 10560 & 13 & 163.6 \\
		\bottomrule
	\end{tabular}
\end{wraptable} 
In this study, we use 3 different datasets with characteristics summarized in Table~\ref{table:data_sum}. The first two datasets are obtained from our industry collaborator (proprietary) and a third public one which we put together. 



\riskd:~ This dataset contains online documents, e.g. news articles, social network posts, which are labeled risky or non-risky to the corporate entities mentioned. If a document is risky, it is further assigned to one of the 15 risk subclasses: \{Climate change, Cyber attack, Data leak, Drug abuse, Engine failure, Fraud, Gun violation, Low stock, Military attack, Misleading statement, Money laundering, Negative growth, Sexual assault, Spying, Trade war\}.

\risks:~ This contains labeled sentences attracted from news articles and categorized into 8 different subclasses: \{Bankruptcy, Bribery corruption, Counterfeiting, Cyber privacy, Environment, Fraud false claims, Labor, Money laundering\}. The majority class consists of non-risky sentences. Note that this dataset comes from a set of articles different from \riskd.

\nyt:~ Extracted from the New York Times, this dataset is comprised of articles on the topics of disasters, i.e. both natural and human-instigated. These topics cover 13 disasters from \{Drought, Earthquakes, Explosions, Floods, Forest and bush fire, Hazardous and toxic substance, Landslides, Lighting, Snowstorms, Tornado, Tropical storm, Volcanoes, Water pollution\}. It also includes a class of random non-disaster news articles from New York Times.

{\bf Document representations.} In this study we apply our work to document datasets, for which we need to define a feature representation. There are numerous options. We report results with tfidf with top 1K words based on frequency, as well as PCA- and ICA-projected data. Linear embedding techniques reduce dimensionality while preserving interpretability. We omit results using non-linear feature representations (e.g., doc2vec \cite{le2014distributed}) as they did not provide any significant performance gain despite computational overhead.

%



{\bf Train/Test Splits.} For each dataset, we randomly partition 2/3 of rare subclasses as seen 
and 1/3 of them as unseen. For training, we use 80\% of the seen subclass instances at random and the rest 20\% forms a seen subclass test set, denoted by $\mR_s$. The set of unseen subclass instances, denoted by $\mR_u$, goes into the test as well. Thus, the rare subclass test consists of 2 parts, i.e. $\mRt = \mR_s \cup \mR_u$. In addition, we also reserve a random 80\% of the majority class for training and the rest 20\% for testing, denoted by $\mN_{test}$. Further, to obtain stable results, we repeat our experiment on 5 different random train/test constructions and report averaged outcomes. 

{\bf Performance Metrics.} (1) For measuring top-level classification performance, we use 3 common metrics \cite{xu2019open,mu2017classification}: \precision, \recall, \fo formally defined in our context as:
{\footnotesize{
$$ \precision = \frac{|\mRt \cap \mRth|}{|\mRth|}, \recall = \frac{|\mRt \cap \mRth|}{|\mRt|}, \fo = \frac{2*\precision*\recall}{\precision + \recall},$$}}
where $\mRth$ is the set of examples predicted as rare subclass. To identify which part of the test (seen or unseen subclasses) the model makes mistakes, we also measure \preseen, \reseen and \reunseen defined as follows:
{\footnotesize{
\begin{gather*}
	\preseen = \frac{|\mR_s \cap \mRth|}{|\mRth|}, \reseen = \frac{|\mR_s \cap \mRth|}{|\mR_s|}, \reunseen = \frac{|\mR_u \cap \mRth|}{|\mR_u|}.
\end{gather*}}}
(2) For sub-level classification test, to quantify the fraction of seen subclass test instances correctly classified and unseen subclass test instances as emerging, we use the following metric:
{\footnotesize{
$$\rarerate = \frac{\overbrace{|\mR_u \cap \mRh_u|}^{acc(emerging)} + \sum_{k=1}^{K}\overbrace{|\mR_{ks} \cap \mRh_{ks}|}^{acc(recurrent)}}{|\mRt|},$$}}
where $\mR_{ks}$ is the set of test examples in subclass $k$ and $\mRh_{ks}$ is the set of examples assigned to that subclass. Here \rarerate$ = 1$ if both seen subclass test instances are perfectly classified to their respective subclasses and unseen subclass instances as emerging. For all of the above metrics, the higher is better.

{\bf Compared Methods.} We compare \method with 2 state-of-the-art methods and 2 simple baselines:
\begin{itemize}
	\item \method-1K, \method-PCA, \method-ICA: 3 versions of \method when tfidf with 1K word dictionary, PCA and ICA representations are used. In \method-PCA, we drop the feature correlation terms $\big(  \bx_{[p]}^T \bx_{[q]} \big)^2$ since features are orthogonal. For the sub-level classification, \method learns a rejection threshold for each specialized classifier based on extreme value theory \cite{siffer2017anomaly}.
	\item \lac \cite{xu2019open}: the most recent method (2019) in open-world classification setting that is based on deep neural networks. We use the recommended parameters $k = 5, n = 9$ (in their notation) from the paper. 
	\item \senc \cite{mu2017classification}: another state-of-the-art ensemble method (2017) using random decision trees for classification under emerging classes. We run \senc with $100$ trees and subsample size $100$ as suggested in their paper.
	\item \bl: a baseline of \method when both cross- and self-correlation terms in Eq.~(\ref{eq:loss}) are removed, via setting $\mu = 0$. Basically, \bl is independently trained $f_0, \dots, f_K$.
	\item \bl-r: a variant of \bl when classification threshold (0.5 by default) is chosen so that the \recall matches that of \method.
\end{itemize}
Note that \senc and \lac aim to detect \textit{any} emerging class without categorizing into rare or majority. For fair comparison, we only inlcude new \textit{rare} subclasses in our test data and consider their rejected instances as belonging to those. In reality, however, emerging classes need to be categorized as rare or not, which these existing methods did not address. 

\begin{table}[!h]
	\begin{adjustwidth}{-.6in}{-.5in}
		\scriptsize
		\centering
		\caption{Performance of methods on the three datasets.}
		\label{table:pre_re_f1}
		\begin{tabular}{l|r|r|r||r|r|r||r|r|r}
			\toprule
			& \multicolumn{3}{c||}{\textbf{Precision}} & \multicolumn{3}{c||}{\textbf{Recall}} & \multicolumn{3}{c}{\textbf{F1}} \\
			\cmidrule{2-10}
			\textbf{Methods} & \textbf{\riskd} & \textbf{\risks} & \textbf{\nyt} & \textbf{\riskd} & \textbf{\risks} & \textbf{\nyt} & \textbf{\riskd} & \textbf{\risks} & \textbf{\nyt} \\
			\hline
			\senc & 0.46$\pm$0.12 & 0.14$\pm$0.03 & 0.36$\pm$0.09 & 0.59$\pm$0.09 & 0.41$\pm$0.08 & 0.39$\pm$0.10 & 0.52$\pm$0.11 & 0.21$\pm$0.04 & 0.37$\pm$0.06 \\
			\lac & 0.79$\pm$0.08 & 0.47$\pm$0.06 & 0.31$\pm$0.07 & 0.57$\pm$0.29 & 0.85$\pm$0.05 & 0.76$\pm$0.07 & 0.63$\pm$0.24 & 0.60$\pm$0.04 & 0.44$\pm$0.07 \\
			\bl & 0.89$\pm$0.04 & 0.79$\pm$0.05 & 0.86$\pm$0.03 & 0.52$\pm$0.06 & 0.55$\pm$0.04 & 0.63$\pm$0.03 & 0.65$\pm$0.05 & 0.65$\pm$0.03 & 0.73$\pm$0.03 \\
			\bl-r & 0.76$\pm$0.10 & 0.56$\pm$0.07 & 0.71$\pm$0.07 & 0.58$\pm$0.13 & 0.57$\pm$0.14 & 0.79$\pm$0.04 & 0.65$\pm$0.10 & 0.55$\pm$0.06 & 0.75$\pm$0.04 \\
			\hline\hline
			\method-1K & 0.89$\pm$0.02 & 0.83$\pm$0.09 & 0.84$\pm$0.03 & 0.58$\pm$0.13 & 0.57$\pm$0.14 & 0.79$\pm$0.04 & 0.70$\pm$0.09 & 0.66$\pm$0.08 & \textBF{0.81$\pm$0.02}  \\
			\method-PCA & 0.85$\pm$0.06 & 0.79$\pm$0.10 & 0.84$\pm$0.10 & 0.58$\pm$0.17 & 0.71$\pm$0.17 & 0.80$\pm$0.07 & 0.68$\pm$0.14 & 0.73$\pm$0.09 & \textBF{0.81$\pm$0.01}  \\ 
			\method-ICA & 0.74$\pm$0.07 & 0.72$\pm$0.11 & 0.84$\pm$0.12 & 0.73$\pm$0.24 & 0.85$\pm$0.18 & 0.82$\pm$0.08 & \textBF{0.71$\pm$0.09} & \textBF{0.78$\pm$0.07} & \textBF{0.81$\pm$0.04} \\
			\bottomrule
		\end{tabular}
	\end{adjustwidth}
\end{table}
\begin{figure}[!h]
	\centering
	\includegraphics[width=0.85\linewidth]{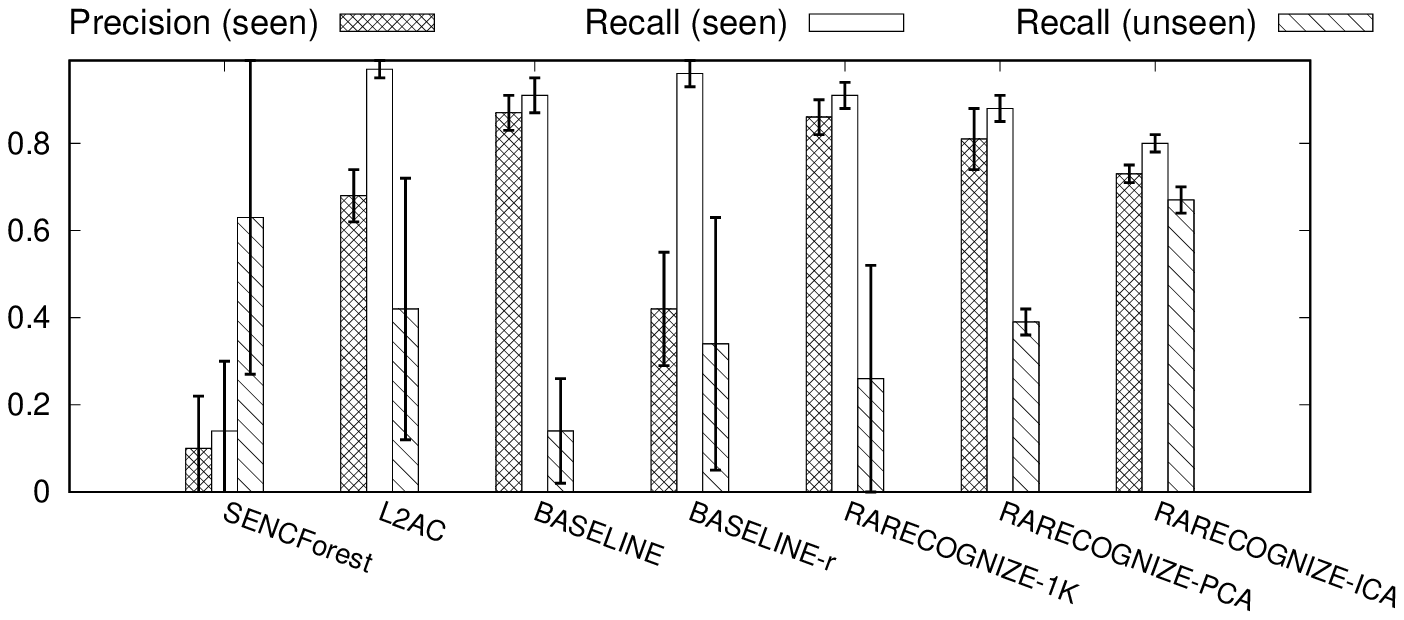}
	\caption{\preseen, \reseen and \reunseen of methods on \riskd (\method-ICA achieves the best balance between \precision and \recall on both seen and unseen test instances).}
	\label{fig:recall_seen_unseen}
\end{figure}

\subsection{Experiment Results}
In the following, we sequentially answer the questions by analyzing our experimental results and comparing between methods.

{\bf RQ1) Top-level Classification into Rare vs. Majority Class.}
We report the \precision, \recall and \fo of all methods on three datasets in Table~\ref{table:pre_re_f1}. For \senc, \lac, \bl and \bl-r, we report results for the representation (tfidf top-1K, PCA, ICA) that yielded the highest \fo value.

From Table~\ref{table:pre_re_f1}, we see that \method-1K, \method-PCA and \method-ICA outperform other methods in terms of \fo score in all cases and \precision, \recall in most cases. Compared to \bl and \bl-r, \fo score of \method-ICA is 6-13\% higher than the best result among the two. This demonstrates that cross- and self-correlations are crucial in \method. Surprisingly, the gap to \senc and \lac is even larger in terms of \fo, between 8-37\% higher. This shows that previous methods on detecting any new emerging classes do not work well when we only target rare subclasses.

Among the three versions of \method, \method-ICA gives the highest \fo. \method-ICA achieves the best balance between precision and recall while \method-1K and \method-PCA seem to have very high \precision but much lower \recall. That means that \method-1K and \method-PCA are better than \method-ICA at discarding majority samples and worse at recognizing rare subclasses.

In Fig.~\ref{fig:recall_seen_unseen}, we have the \preseen, \reseen and \reunseen measures of all the methods on \riskd (Figures for other datasets are similar, see Supp. A.2). This figure shows that \method-ICA also achieves a good balance between seen and unseen subclass classification, i.e., it recognizes both these subclasses equally well. On the other hand, most of other methods achieve high \preseen and \reseen but much lower \reunseen, except \senc which only has high \reunseen. This is because \senc rejects most instances as unseen which however hurts \precision drastically.

\begin{table}[!h]
	\begin{adjustwidth}{-.5in}{-.3in}
	\centering
	\caption{\rarerate of methods on the three datasets.}
	\label{table:unseen_rate}
	\begin{tabular}{l| r|r|r}
		\toprule
		\textbf{Methods} & \textbf{\riskd} & \textbf{\risks} & \textbf{\nyt} \\
		\hline
		\senc & 0.37$\pm$0.09  & 0.41$\pm$0.08  & 0.34$\pm$0.04  \\
		\lac & 0.22$\pm$0.17 & 0.20$\pm$0.20  & 0.08$\pm$0.12  \\
		\bl & 0.41$\pm$0.07 & 0.37$\pm$0.04 & 0.41$\pm$0.04 \\
		\bl-r & 0.43$\pm$0.16 & 0.38$\pm$0.03 & 0.42$\pm$0.04 \\
		\hline\hline
		\method-1K & 0.45$\pm$0.08 & 0.38$\pm$0.12 & 0.46$\pm$0.12 \\
		\method-PCA & 0.50$\pm$0.14 & 0.59$\pm$0.14 & 0.65$\pm$0.15 \\ 
		\method-ICA & \textBF{0.63$\pm$0.14} & \textBF{0.62$\pm$0.09} & \textBF{0.64$\pm$0.15} \\ 
		\bottomrule
	\end{tabular}
	\end{adjustwidth}
\end{table}
{\bf RQ2) Sub-level Classification into Recurrent and Emerging Rare Subclasses.}
We report the \rarerate of all the methods in Table~\ref{table:unseen_rate} (Breakdown of errors in confusion tables  are given in Supp. A.3).

Tables \ref{table:pre_re_f1} and \ref{table:unseen_rate} reflect that all three versions of \method are always better or comparable to the others in terms of \rarerate and \method-ICA achieves the highest value. \method-ICA achieves significantly higher \rarerate than all the baselines. \senc seems to perform the next best due to the fact that it classifies most of the instances as emerging which results in high classification performance on unseen subclasses.

\begin{figure}[!h]
	\captionsetup[subfigure]{justification=centering}
	\begin{subfigure}[b]{0.25\linewidth}
		\centering
		\includegraphics[width=0.9\linewidth]{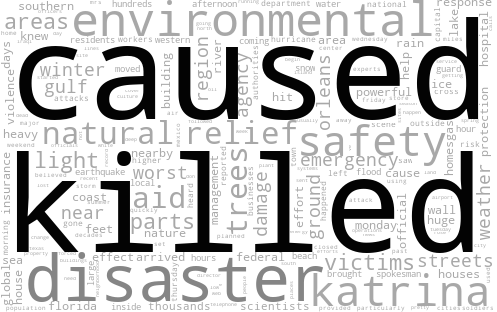}
		\caption{General disaster} 
		\label{fig7:a} 
	\end{subfigure}
	\rulesep
	\begin{subfigure}[b]{0.25\linewidth}
		\centering
		\includegraphics[width=0.9\linewidth]{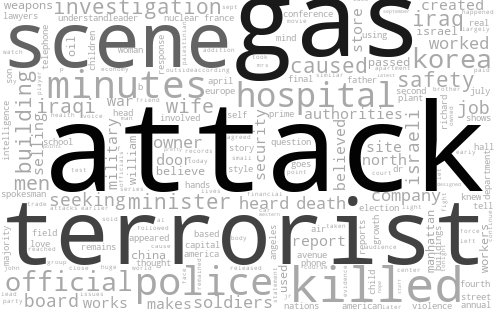}
		\caption{Explosions} 
		\label{fig7:b} 
	\end{subfigure} 
	\begin{subfigure}[b]{0.25\linewidth}
		\centering
		\includegraphics[width=0.9\linewidth]{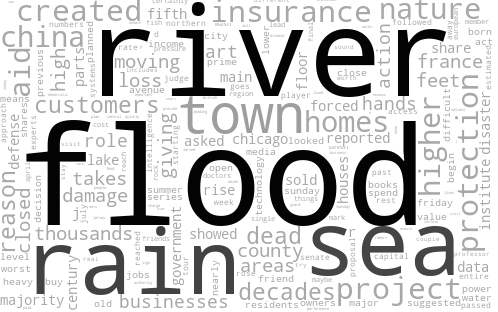}
		\caption{Floods} 
		\label{fig7:c} 
	\end{subfigure}
	\begin{subfigure}[b]{0.25\linewidth}
		\centering
		\includegraphics[width=0.9\linewidth]{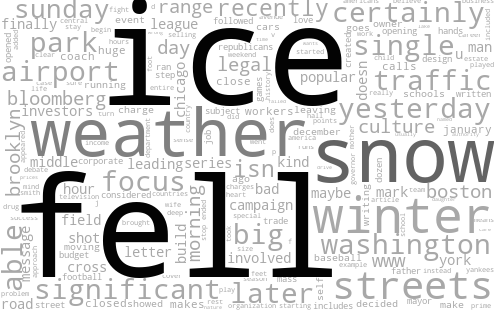}
		\caption{Snowstorms} 
		\label{fig7:d} 
	\end{subfigure}
	\caption{Word clouds for the weights of general and 3 specialized classifiers in \method-1K on \nyt (See Supp. A.4 for other subclasses).}
	\label{fig:word_cloud}
\end{figure}


{\bf RQ3) Model Interpretation.}
In Fig.~\ref{fig:word_cloud}, we plot the wordclouds representing the general and 3 specialized classifiers for 3 disaster subclasses (sizes of the words proportional to their weights learned by \method-1K). Existing methods, \senc and \lac, are not interpretable due to respective ensemble and deep neural network-based models they employ.

From Fig.~\ref{fig:word_cloud}, specialized disaster classifiers are clearly characterized by specific words closely related to the respective disasters, whereas the general classifier is heavily weighted by common words to every disaster. Specifically, Explosions classifier picks up \texttt{attack, gas, terrorist, scene} as most weighted keywords, and Snowstorms classifier puts heavy weights on words \texttt{ice, fell, snow, weather}. The general classifier is highlighted by the words \texttt{caused, killed, disaster} which describe consequences of most disasters. Wordclouds on other disaster subclasses, along with those for other datasets are in Supp. A.4.

Thanks to the interpretability that \method offers, we can look deeper into the significance of individual words in classifying documents. Besides its promising quantitative performance, these qualitative results confirm that our method has learned what agrees with human intuition.
\begin{figure}[!h]
	\centering
	\begin{minipage}[b]{0.495\linewidth}
		\centering
		\includegraphics[width=1\linewidth]{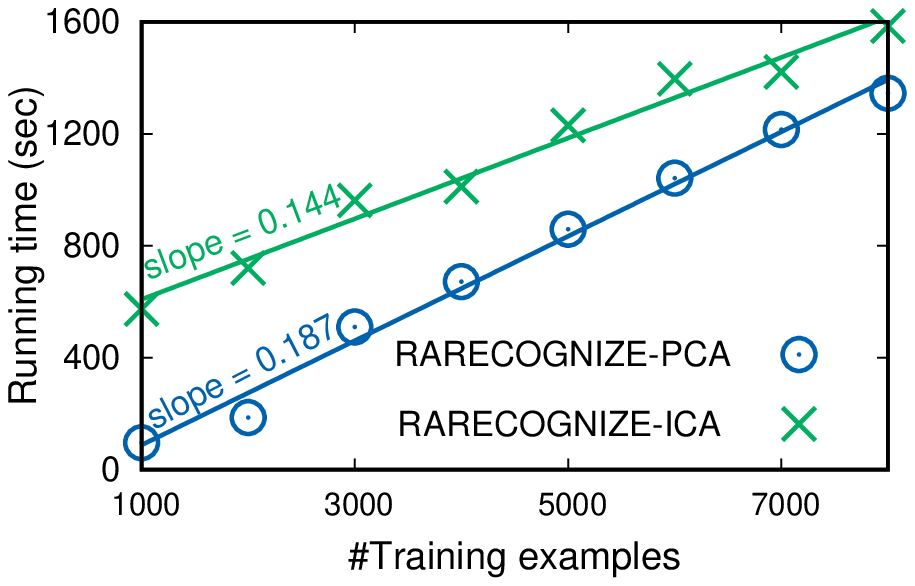}
		\caption{\method~scales linearly in number of training examples.} 
		\label{fig:scalability} 
	\end{minipage} \hfill
	\begin{minipage}[b]{0.495\linewidth}
		\centering
		\includegraphics[width=1\linewidth]{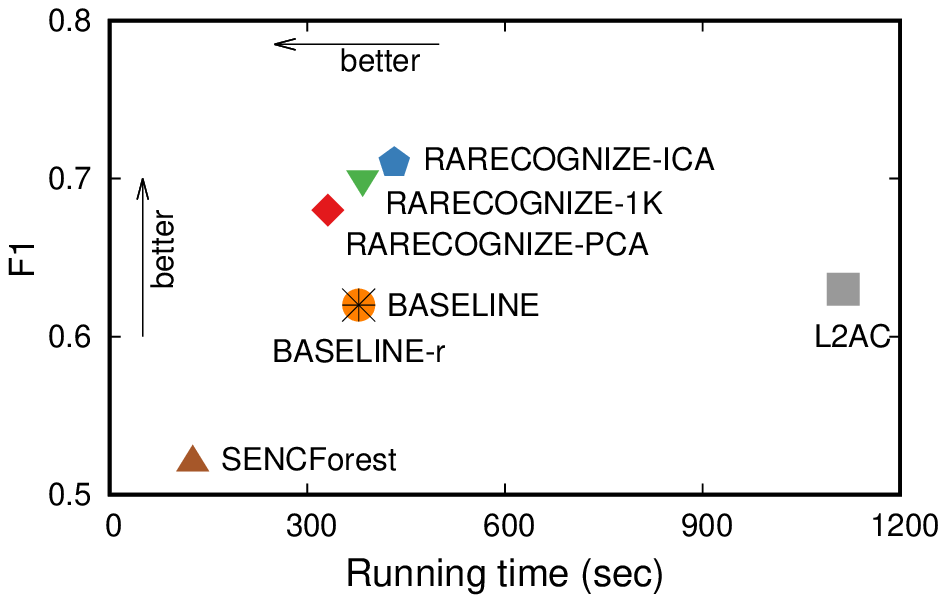}
		\caption{Time-Performance trade-off of all compared methods.} 
		\label{fig:time_quality} 
	\end{minipage}
\end{figure}

{\bf RQ4) Scalability and Time-Performance Trade-off.}
Besides our formal complexity analysis, we demonstrate the scalability of \method empirically. 
Fig.~\ref{fig:scalability} shows the running time of \method-PCA and \method-ICA when varying the amount of training data. The running time increases linearly with the data size. \method with PCA is faster than that with ICA thanks to no feature correlations (i.e. $\big(  \bx_{[p]}^T \bx_{[q]} \big)^2$ dropped).

In Fig.~\ref{fig:time_quality}, we show the time-performance trade-off among all compared methods. We conclude that \method with three representations run relatively fast, only slower than \senc, and returns the highest performance in terms of \fo. \lac consumes a huge amount of time for training a neural network, with subpar performance.

\section{Related Work}
\label{sec:relatedwork}
Our work is closely related with two fields, namely open-world classification and continual learning. Both belong to the category of lifelong machine learning \cite{chen2016lifelong}.

\textbf{Open-world classification.} 
Traditional close-world classification assumes that all test classes are known and seen in training data \cite{kim2014convolutional,zhang2015character}. However, such assumption could be violated in reality. Open-world classification, in contrast, assumes unseen and novel classes could emerge during test time, and addresses the classification problem by recognizing unseen classes. Previous works
\cite{shu2017doc,shu2018unseen,xu2019open,mu2017streaming,mu2017classification} propose different approaches under this setting.

Specifically, DOC
\cite{shu2017doc} leverages convolutional neural nets (CNNs) with multiple sigmoid functions to classify examples as seen or emerging. \cite{shu2018unseen} follows the same DOC module and performs hierarchical clustering to all rejected samples. Later, L2AC
\cite{xu2019open} proposes to use a meta-classifier and a ranker to add or delete a class without re-training. However, it requires a large amount of computation in both training and testing due to the top-$k$ search over all training data. 

SENCForest \cite{mu2017classification} is a randomized ensemble method. It grows multiple random forests and rejects examples when all random forests yield "new class". Under the same setting, SENC-MaS \cite{mu2017streaming} maintains matrix sketchings to decide whether an example belongs to a seen class or emerging.

In the emerging rare subclass setting, the previous approaches aim at recognizing \textit{any and every} classes and are not able to ignore the \emph{not-of-interest} classes while recognizing emerging ones, thus consume much more memory and time.

\textbf{Continual learning.}
There are recent works investigating continual learning or incremental learning \cite{shin2017continual,rebuffi2017icarl}. They aim at solving the issue of catastrophic forgetting \cite{french} in connectionist networks. In this field, models are proposed to continually learn new classes without losing performance on old seen classes.




Previous works \cite{rebuffi2017icarl,kirkpatrick2017overcoming,lee2017overcoming,kemker2018fearnet} show promising results. However, the number of documents in rare subclasses \emph{of-interest} in our setting is usually not large enough for neural networks to be sufficiently trained. Consequently, the neural network approach does not perform well in rare-class classification and recognition.

\section{Conclusion}
\label{sec:conclude}
We proposed \method for rare-class recognition over a continuous stream, in which new subclasses may emerge. \method employs a general classifier to filter out not-rare class instances (top-level) and a set of specialized classifiers that recognize known rare subclasses or otherwise reject as emerging (sub-level). Since majority of incoming instances are filtered out and new rare subclasses are a few, \method processes incoming data fast and grows in size slowly. Extensive experiments show that it outperforms two most recent state of the art as well as two simple baselines significantly in both top- and sub-level tasks, while achieving the best efficiency-performance balance and offering interpretability.
Future work will extend \method to an end-to-end system that clusters emerging instances and trains all the relevant models incrementally.

\hide{
\section*{Acknowledgments}\\
This research is sponsored by PricewaterhouseCoopers Risk and Regulatory Services Innovation Center at Carnegie Mellon University as well as NSF CAREER 1452425. 
Any opinions, findings, and recommendations expressed in this
material are those of the author(s) and do not necessarily reflect the views of
the funding parties.
}

{\small{
\bibliographystyle{abbrv}

}}

\newpage
\appendix
\section{Supplementary}
\label{sec:appendix}

\subsection{Proof of Theorem 1.}
\label{sup:proof_1}
Given that $\ell(\cdot)$ and L-$p$ norms for $p\geq 1$ are convex, and that sum of non-negative convex functions remains convex, it suffices to show that the correlation term (denoted by $C$) of the loss function $\mL$ is convex.
\beq
\label{overallC}
C =  \sum_{p,q} 
\bigg\{ 
{\frac{\mu}{4}
	(w^2_{0,p} w^2_{0,q})
	+ \frac{\mu}{4} \sum_{k'=1}^{K} (w^2_{k',p} w^2_{k',q})}
+ \frac{\mu}{2} {w^2_{0,p} (\sum_{k'=1}^{K} w^2_{k',q})}
\bigg\} 
\big(  \bx_{[p]}^T \bx_{[q]} \big)^2
\eeq

We prove convexity by showing that the Hessian matrix of $C$ is positive semi-definite (PSD). The Hessian in this case is a $d(K+1)\times d(K+1)$ matrix, denoted $\mathbf{H}$, containing all the second-order derivatives as illustrated in the following drawing. The diagonal $(d\times d)$ matrices correspond to self-correlation derivatives $\frac{\partial C}{\partial w_{0,z}\partial w_{0,t}}$ and $\frac{\partial C}{\partial w_{k',z}\partial w_{k',t}}$ for $k'=1\ldots K$. The $(d\times d)$ off-diagonals contain cross-correlation derivatives   $\frac{\partial C}{\partial w_{0,z}\partial w_{k',t}}$ for $k'=1\ldots K$.
\begin{center}
	\begin{tikzpicture}[
	box/.style={draw,rectangle,minimum size=1.5cm,text width=1.2cm,align=center}]
	\matrix (conmat) [row sep=.1cm,column sep=.1cm] {
		\node (one) [box,
		label=left:\( {d} \),
		label=above:\( {d} \),
		] {$\frac{\partial C}{\partial \bw_{0}\partial \bw_{0}}$};
		&
		\node (two) [box,
		label=above: {$d$}] {$\frac{\partial C}{\partial \bw_{0}\partial \bw_{1}}$};
		&
		\node (thr) [
		] {$\ldots$};
		&
		\node (four) [box,
		label=above: $d$] {$\frac{\partial C}{\partial \bw_{0}\partial \bw_{K}}$};
		
		\\
		
		\node (five) [box,
		label=left: $d$] {$\frac{\partial C}{\partial \bw_{1}\partial \bw_{0}}$};
		&
		\node (six) [box] {$\frac{\partial C}{\partial \bw_{1}\partial \bw_{1}}$};
		&
		\node (sev) [
		] {$\ldots$};
		&
		\node (eigh) [box] {$\frac{\partial C}{\partial \bw_{1}\partial \bw_{K}}$};
		\\	
			\node (nine) [] {$\vdots$}; & \node (ten) [
		] {$\vdots$}; & \node (ele) [
		] {$\ddots$}; & \node (tw) [
		] {$\vdots$};
		\\
		\node (five) [box,
		label=left: $d$] {$\frac{\partial C}{\partial \bw_{K}\partial \bw_{0}}$};
		&
		\node (six) [box] {$\frac{\partial C}{\partial \bw_{K}\partial \bw_{1}}$};
		&
		\node (sev) [
		] {$\ldots$};
		&
		\node (eigh) [box] {$\frac{\partial C}{\partial \bw_{K}\partial \bw_{K}}$};\\
	};
	\node [rotate=0,left=.5cm of conmat,anchor=center,text width=1.5cm,align=center] {\textbf{H = }};
	\end{tikzpicture}
\end{center}
We will derive the above three types of terms in the following.
$$
\text{\bf 1)} \;\;\; \frac{\partial C}{\partial \bw_{0}\partial \bw_{0}} \text{:}
$$

\beq
\text{ }\frac{\partial C}{\partial w_{0,z}\partial w_{0,t}} = 
\frac{\partial }{\partial w_{0,z}\partial w_{0,t}}
\sum_{p,q} 
\bigg\{ 
 \underbrace{\frac{\mu}{4} (w^2_{0,p} w^2_{0,q}) \big(  \bx_{[p]}^T \bx_{[q]} \big)^2}_{Z_1}
+ \underbrace{\frac{\mu}{2} \sum_{k'=1}^{K} (w^2_{0,p} w^2_{k',q})\big(  \bx_{[p]}^T \bx_{[q]} \big)^2}_{Z_2}
\bigg\} 
\eeq
which excludes the terms in Eq. \eqref{overallC} that do not depend on $\bw_0$.

\beq
\frac{\partial Z_1}{\partial w_{0,z}\partial w_{0,t}} = 
\left\{\begin{array}{lr}
	2\mu\; w_{0,z} w_{0,t} \;(\bx_{[z]}^T \bx_{[t]})^2 & \;\; \text{if } t\neq z\\
	&\\
	2\mu\; w^2_{0,z}\; (\bx_{[z]}^T \bx_{[z]})^2 
	+ \mu\; \sum_q w^2_{0,q}\; (\bx_{[z]}^T \bx_{[q]})^2 
	& \;\; \text{if } t = z\\
\end{array}\right\}
\eeq

\beq
\frac{\partial Z_2}{\partial w_{0,z}\partial w_{0,t}} = 
\left\{\begin{array}{lr}
	0 & \;\; \text{if } t\neq z\\
	&\\
	 \mu\; \sum_q \sum_{k'=1}^K w^2_{k',q}\; (\bx_{[z]}^T \bx_{[q]})^2 
	& \;\; \text{if } t = z\\
\end{array}\right\}
\eeq

Let us denote $\bW = \bw_0 \bw_0^T$ and $\bG = \bX^T\bX$. Then,
\beq
\Rightarrow \boxed{\frac{\partial C}{\partial \bw_{0}\partial \bw_{0}} =
2\mu  \; \bigg[ \bW  \odot \bG \odot \bG \bigg] + \text{diag}(v^{(0)}_1,\ldots, v^{(0)}_d)}
\eeq
where $v^{(0)}_p = \frac{1}{2} \sum_q (w^2_{0,q} + \sum_{k'=1}^K w^2_{k',q})
(\bx_{[p]}^T \bx_{[q]})^2 \geq 0$.

$$
\text{\bf 2)} \;\;\; \frac{\partial C}{\partial \bw_{k}\partial \bw_{k}} \text{:}
$$

\beq
\text{ }\frac{\partial C}{\partial w_{k,z}\partial w_{k,t}} = 
\frac{\partial }{\partial w_{k,z}\partial w_{k,t}}
\sum_{p,q} 
\bigg\{ 
\underbrace{\frac{\mu}{4} (w^2_{k,p} w^2_{k,q}) \big(  \bx_{[p]}^T \bx_{[q]} \big)^2}_{K_1}
+ \underbrace{\frac{\mu}{2} (w^2_{0,p} w^2_{k,q})\big(  \bx_{[p]}^T \bx_{[q]} \big)^2}_{K_2}
\bigg\} 
\eeq
which excludes the terms in Eq. \eqref{overallC} that do not depend on $\bw_k$.

\beq
\frac{\partial K_1}{\partial w_{k,z}\partial w_{k,t}} = 
\left\{\begin{array}{lr}
	2\mu\; w_{k,z} w_{k,t} \;(\bx_{[z]}^T \bx_{[t]})^2 & \;\; \text{if } t\neq z\\
	&\\
	2\mu\; w^2_{k,z}\; (\bx_{[z]}^T \bx_{[z]})^2 
	+ \mu\; \sum_q w^2_{k,q}\; (\bx_{[z]}^T \bx_{[q]})^2 
	& \;\; \text{if } t = z\\
\end{array}\right\}
\eeq

\beq
\frac{\partial K_2}{\partial w_{k,z}\partial w_{k,t}} = 
\left\{\begin{array}{lr}
	0 & \;\; \text{if } t\neq z\\
	&\\
	\mu\; \sum_p  w^2_{0,p}\; (\bx_{[p]}^T \bx_{[z]})^2 
	& \;\; \text{if } t = z\\
\end{array}\right\}
\eeq

Then, for $v^{(k)}_p = \frac{1}{2} \sum_q (w^2_{0,q} + w^2_{k,q})
(\bx_{[p]}^T \bx_{[q]})^2 \geq 0$ we can write
\beq
\Rightarrow \boxed{\frac{\partial C}{\partial \bw_{k}\partial \bw_{k}} =
	2\mu  \; \bigg[ \bw_k\bw_k^T  \odot \bG \odot \bG \bigg] + \text{diag}(v^{(k)}_1,\ldots, v^{(k)}_d)}
\eeq

$$
\text{\bf 3)} \;\;\; \frac{\partial C}{\partial \bw_{0}\partial \bw_{k}} \text{:}
$$

\beq
\text{ }\frac{\partial C}{\partial w_{0,z}\partial w_{k,t}} = 
\frac{\partial }{\partial w_{0,z}\partial w_{k,t}}
\sum_{p,q} 
\bigg\{ 
\underbrace{\frac{\mu}{2} (w^2_{0,p} w^2_{k,q})\big(  \bx_{[p]}^T \bx_{[q]} \big)^2}_{T}
\bigg\} 
\eeq
which excludes the terms in Eq. \eqref{overallC} that do not depend on both $\bw_0$ and $\bw_k$.

\beq
\frac{\partial T}{\partial w_{0,z}\partial w_{k,t}} = 
2\mu\; w_{0,z} w_{k,t} \; (\bx_{[z]}^T \bx_{[t]})^2 
\Rightarrow \boxed{\frac{\partial C}{\partial \bw_{0}\partial \bw_{k}} =
	2\mu  \; \bigg[ \bw_0\bw_k^T  \odot \bG \odot \bG \bigg] }
\eeq

Let us denote $\bD_0 = \text{diag}(v^{(0)}_1,\ldots, v^{(0)}_d)$,
$\bD_k = \text{diag}(v^{(k)}_1,\ldots, v^{(k)}_d)$, and $\bG^2 = \bG \odot \bG$. Then the Hessian matrix can be written as the sum of the following matrices.

\beq
\mathbf{H} = 
\begin{bmatrix}[c|c|c|c|c]
	\bw_0 \bw_0^T \odot \bG^2 & \bw_0 \bw_1^T \odot \bG^2 & \bw_0 \bw_2^T \odot \bG^2 & \dots  & \bw_0 \bw_K^T \odot \bG^2 \\ \hline
	\bw_1 \bw_0^T \odot \bG^2 & \bw_1 \bw_1^T \odot \bG^2 & \bzero&  \dots  & \bzero \\\hline
	\vdots & \vdots  & \vdots  & \ddots & \vdots \\\hline
	\bw_K \bw_0^T \odot \bG^2 & \bzero & \dots & \bzero & \bw_K \bw_K^T \odot \bG^2 \\
\end{bmatrix}
+
\begin{bmatrix}
	\bD_0 & \bzero & \bzero & \dots  & \bzero \\
	\bzero & \bD_1 & \bzero&  \dots  & \bzero \\
	\vdots & \vdots  & \vdots  & \ddots & \vdots \\
	\bzero & \bzero & \dots & \bzero & \bD_K \\
\end{bmatrix}
\eeq

The second diagonal matrix, denoted $\mathbf{H}_2$, is PSD as it contains non-negative entries $v^{(k)}_p = \frac{1}{2} \sum_q (w^2_{0,q} + w^2_{k,q})
(\bx_{[p]}^T \bx_{[q]})^2$, $p=1\ldots d$ for $k=1,\ldots, K$ and non-negative entries $v^{(0)}_p = \frac{1}{2} \sum_q (w^2_{0,q} + \sum_{k'=1}^K w^2_{k',q})
(\bx_{[p]}^T \bx_{[q]})^2$, $p=1\ldots d$. 

We can also show that the first matrix, denoted $\mathbf{H}_1$, is PSD by first decomposing it into
a sum of outer products between vectors of the form
\begin{align}
&\bG^2 \odot (
\begin{bmatrix} 
	\frac{\bw_0}{\sqrt{(K+1)}}  \\
	\bw_1 \\
	\bzero\\
	$\vdots$ \\
	\bzero
\end{bmatrix}
\begin{bmatrix} 
	\frac{\bw_0}{\sqrt{(K+1)}}  &
	\bw_1^T &
	\bzero^T&
	$\ldots$ &
	\bzero^T
\end{bmatrix}
+ 
\begin{bmatrix} 
	\frac{\bw_0}{\sqrt{(K+1)}}  \\
	\bzero\\
	\bw_2 \\
	$\vdots$ \\
	\bzero
\end{bmatrix}
\begin{bmatrix} 
	\frac{\bw_0}{\sqrt{(K+1)}}  &
	\bzero^T&
	\bw_2^T &
	$\ldots$ &
	\bzero^T
\end{bmatrix}
) \\
& \ldots + \begin{bmatrix} 
\frac{\bw_0}{\sqrt{(K+1)}}  \\
\bzero\\
$\vdots$ \\
\bzero\\
\bw_K
\end{bmatrix}
\begin{bmatrix} 
\frac{\bw_0}{\sqrt{(K+1)}}  &
\bzero^T&
$\ldots$ &
\bzero^T&
\bw_K^T 
\end{bmatrix}
= \bG^2 \odot (\bv_1\bv_1^T + \ldots + \bv_K\bv_K^T) = \mathbf{H}_1
\end{align}

It is easy to see that $\br^T \mathbf{H}_1 \br \geq 0$ for {\em any} vector $\br$, since $(\br^T \bv_k)(\bv_k^T\br) = s_k^2 \geq 0$ for all $k$ terms that constitute the sum.

We have shown that both $\mathbf{H}_1$ and $\mathbf{H}_2$ are PSD and 
since the sum of two PSD matrices is also PSD, then the Hessian matrix $\mathbf{H}$ is also PSD, which concludes the proof for convexity. \qed

\subsection{Top-level classification: \preseen, \reseen, \reunseen for \risks and \nyt datasets.} 
\label{sup:top_level}
Figure~\ref{fig:sent_recall_seen_unseen} and \ref{fig:nyt_recall_seen_unseen} provide additional results of \preseen, \reseen and \reunseen on two datasets: \risks and \nyt.
\begin{figure}[!h]
	\centering
	\includegraphics[width=0.85\linewidth]{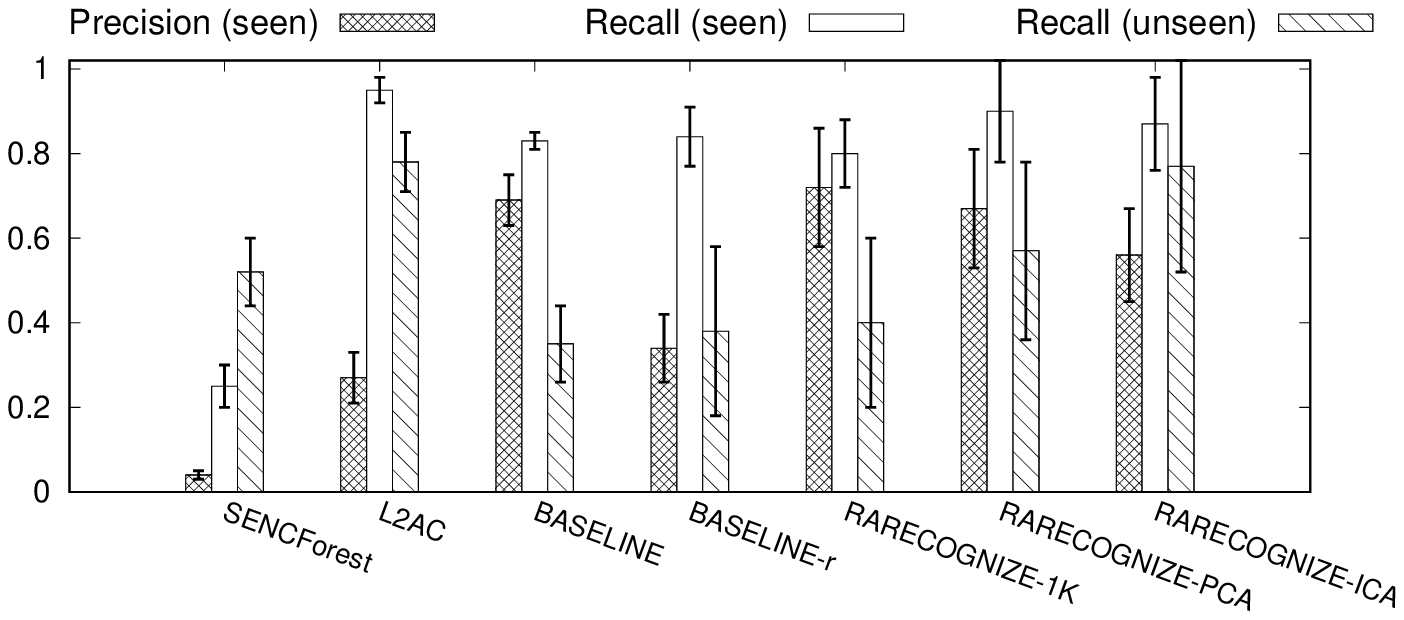}
	\caption{Top-level classification: \preseen, \reseen and \reunseen of methods on \risks dataset.}
	\label{fig:sent_recall_seen_unseen}
\end{figure}
\begin{figure}[!h]
	\centering
	\includegraphics[width=0.85\linewidth]{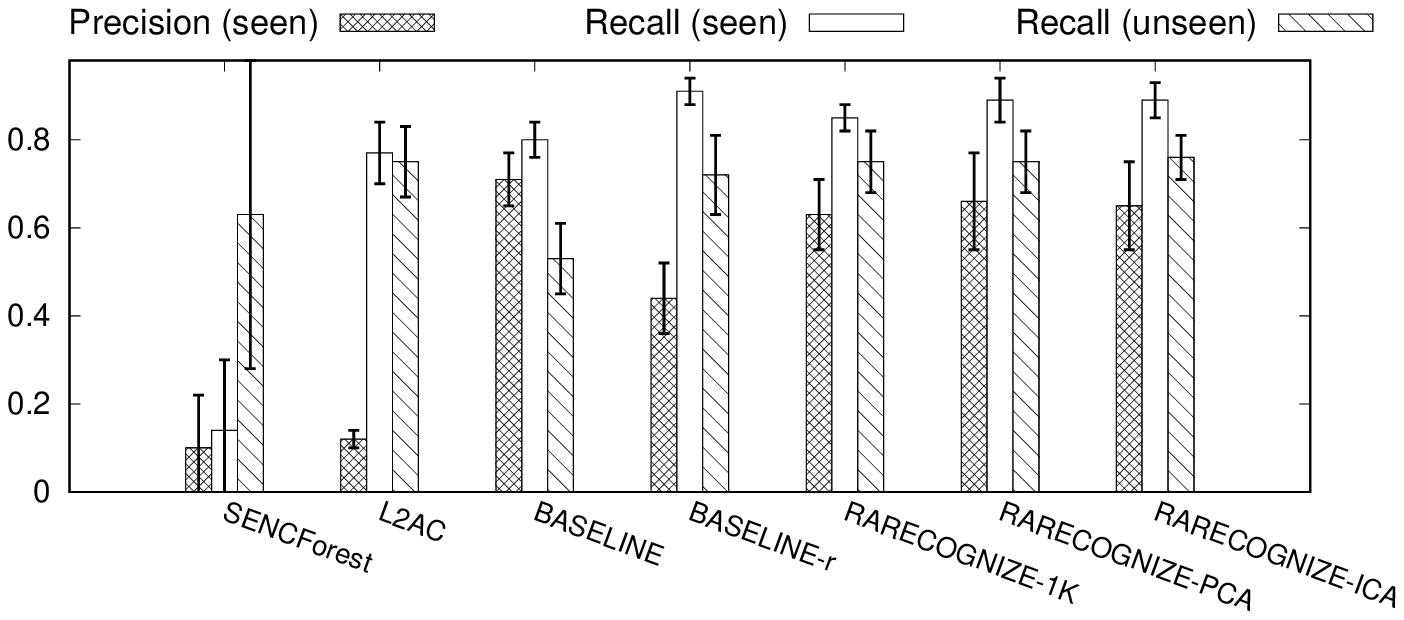}
	\caption{Top-level classification: \preseen, \reseen and \reunseen of methods on \nyt dataset.}
	\label{fig:nyt_recall_seen_unseen}
\end{figure}

\subsection{Sub-level classification: Breakdown confusion tables to compute \rarerate for the three datasets.}
\label{sup:sub_level}
Tables~\ref{table:rare_rate}, \ref{table:rare_rate_sent} and \ref{table:rare_rate_nyt} present confusion tables of \rarerate on \riskd, \risks and \nyt datasets.
\begin{table}[!h]
	\centering
	\caption{Sub-level classification: Breakdown confusion tables to compute \rarerate in \riskd (Bolded values reflect correct predictions).}
	\label{table:rare_rate}
	\begin{subtable}{.38\linewidth}
		\begin{tabular}{l|r|r|r}
			\toprule
			Label & $\mR_{ks}$ & $\mR_{u}$ & $\mN_{test}$ \\
			\midrule
			$\hat \mR_{ks}$ & \textBF{81.8} & \multirow{2}{*}{6.0} & \multirow{2}{*}{38.4} \\
			\cline{1-2}
			$\hat \mR_{k's}(k' \neq k)$ & 8.6 & & \\
			\hline
			$\hat \mR_{u}$ & 301.0 & \textBF{218.8} & 482.4 \\
			\hline
			$\hat \mN$ & 0.8 & 205.6 & \textBF{26.4}\\
			\bottomrule
		\end{tabular}
		\caption[a.]{\senc}
	\end{subtable}
	\begin{subtable}{.22\linewidth}
		\begin{tabular}{r|r|r}
			\toprule
			$\mR_{ks}$ & $\mR_u$ & $\mN_{test}$ \\
			\midrule
			\textBF{40.8} & \multirow{2}{*}{111.6} & \multirow{2}{*}{125.8} \\ \cline{1-1}
			225.6 & & \\
			\hline
			115.6 & \textBF{139.4} & 44.4 \\
			\hline
			10.2 & 179.4 & \textBF{377.0} \\
			\bottomrule
		\end{tabular}
		\caption{\lac}
	\end{subtable}
	\begin{subtable}{.22\linewidth}
		\begin{tabular}{r|r|r}
			\toprule
			$\mR_{ks}$ & $\mR_u$ & $\mN_{test}$ \\
			\midrule
			\textBF{294.0} & \multirow{2}{*}{64.2} & \multirow{2}{*}{46.6}\\\cline{1-1}
			14.0 & & \\
			\hline
			75.0 & \textBF{44.2} & 125.4 \\
			\hline
			9.2 & 322.0 & \textBF{375.2} \\
			\bottomrule
		\end{tabular}
		\caption{\bl}
	\end{subtable}
	\begin{subtable}{.3\linewidth}
		\centering
		\begin{tabular}{r|r|r}
			\toprule
			$\mR_{ks}$ & $\mR_u$ & $\mN_{test}$ \\
			\midrule
			\textBF{294.0} & \multirow{2}{*}{70.4} & \multirow{2}{*}{46.6}\\\cline{1-1}
			14.0 & & \\
			\hline
			75.0 & \textBF{59.2} & 125.6 \\
			\hline
			9.2 & 300.8 & \textBF{375.0} \\
			\bottomrule
		\end{tabular}
		\caption{\bl-r}
	\end{subtable}
	\begin{subtable}{.30\linewidth}
		\centering
		\begin{tabular}{r|r|r}
			\toprule
			$\mR_{ks}$ & $\mR_{u}$ & $\mN_{test}$ \\
			\midrule
			\textBF{293.0} & \multirow{2}{*}{28.0} & \multirow{2}{*}{35.2} \\\cline{1-1}
			 13.2 & & \\
			\hline
			54.0 & \textBF{77.2} & 21.0\\
			\hline
			32.0 & 325.2 & \textBF{491.0} \\
			\bottomrule
		\end{tabular}
		\caption[a.]{\method-1K}
	\end{subtable}
	\begin{subtable}{.30\linewidth}
		\centering
		\begin{tabular}{r|r|r}
			\toprule
			$\mR_{ks}$ & $\mR_u$ & $\mN_{test}$ \\
			\midrule
			\textBF{330.4} & \multirow{2}{*}{22.8} & \multirow{2}{*}{25.0}\\\cline{1-1}
			14.8 & & \\
			\hline
			0.4 & \textBF{81.0} & 48.2 \\
			\hline
			46.6 & 326.4 & \textBF{474.0} \\
			\bottomrule
		\end{tabular}
		\caption{\method-PCA}
	\end{subtable}
	\begin{subtable}{.30\linewidth}
		\centering
		\begin{tabular}{r|r|r}
			\toprule
			$\mR_{ks}$ & $\mR_u$ & $\mN_{test}$ \\
			\midrule
			\textBF{304.0} & \multirow{2}{*}{102.8} & \multirow{2}{*}{220.2}\\\cline{1-1}
			20.0 & & \\
			\hline
			0.0 & \textBF{214.2} & 137.2\\
			\hline
			68.2 & 113.4 & \textBF{189.8}\\
			\bottomrule
		\end{tabular}
		\caption{\method-ICA}
	\end{subtable}
\end{table}
\begin{table}[!h]
	\centering
	\caption{Sub-level classification: Breakdown confusion tables to compute \rarerate in \risks.}
	\label{table:rare_rate_sent}
	\begin{subtable}{.37\linewidth}
		\begin{tabular}{l|r|r|r}
			\toprule
			& $\mR_{ks}$ & $\mR_u$ & $\mN_{test}$ \\
			\midrule
			$\hat \mR_{ks}$ & \textBF{58.0} & \multirow{2}{*}{185.6} & \multirow{2}{*}{769.6} \\
			\cline{1-2}
			$\hat \mR_{k's} (k' \neq k)$ & 80.0 & & \\
			\hline
			$\hat \mR_u$ & 99.6 & \textBF{199.4} & 776.4 \\
			\hline
			$\hat \mN$ & 0.0 & 0.0 & \textBF{0.0}\\
			\bottomrule
		\end{tabular}
		\caption{\senc}
	\end{subtable}
	\begin{subtable}{.24\linewidth}
		\centering
		\begin{tabular}{r|r|r}
			\toprule
			$\mR_{ks}$ & $\mR_u$ & $\mN_{test}$ \\
			\midrule
			\textBF{57.2} & \multirow{2}{*}{236.2} & \multirow{2}{*}{426.6} \\ \cline{1-1}
			98.8 & & \\
			\hline
			70.8 & \textBF{65.6} & 187.8 \\
			\hline
			10.8 & 83.2 & \textBF{931.6} \\
			\bottomrule
		\end{tabular}
		\caption{\lac}
	\end{subtable}
	\begin{subtable}{.24\linewidth}
		\centering
		\begin{tabular}{r|r|r}
			\toprule
			$\mR_{ks}$ & $\mR_u$ & $\mN_{test}$ \\
			\midrule
			\textBF{156.0} & \multirow{2}{*}{68.8} & \multirow{2}{*}{92.0}\\\cline{1-1}
			11.0 & & \\
			\hline
			30.2 & \textBF{73.0} & 207.6 \\
			\hline
			40.4 & 243.2 & \textBF{1246.4} \\
			\bottomrule
		\end{tabular}
		\caption{\bl}
	\end{subtable}
	\begin{subtable}{.3\linewidth}
		\centering
		\begin{tabular}{r|r|r}
			\toprule
			$\mR_{ks}$ & $\mR_u$ & $\mN_{test}$ \\
			\midrule
			\textBF{156.0} & \multirow{2}{*}{74.2} & \multirow{2}{*}{92.0}\\\cline{1-1}
			11.0 & & \\
			\hline
			30.2 & \textBF{81.4} & 207.6 \\
			\hline
			40.4 & 229.4 & \textBF{1246.4} \\
			\bottomrule
		\end{tabular}
		\caption{\bl-r}
	\end{subtable}
	\begin{subtable}{.3\linewidth}
		\centering
		\begin{tabular}{r|r|r}
			\toprule
			$\mR_{ks}$ & $\mR_u$ & $\mN_{test}$ \\
			\midrule
			\textBF{158.0} & \multirow{2}{*}{83.6} & \multirow{2}{*}{36.0}\\
			\cline{1-1}
			10.6 & & \\
			\hline
			22.2 & \textBF{78.0} & 46.6\\
			\hline
			46.8 & 223.4 & \textBF{1463.4}\\
			\bottomrule
		\end{tabular}
		\caption{\method-1K}
	\end{subtable}
	\begin{subtable}{.3\linewidth}
		\centering
		\begin{tabular}{r|r|r}
			\toprule
			$ \mR_{ks}$ & $ \mR_u$ & $ \mN_{test}$ \\
			\midrule
			\textBF{186.0} & \multirow{2}{*}{41.6} & \multirow{2}{*}{52.8}\\ \cline{1-1}
			28.2 & & \\
			\hline
			2.2 & \textBF{180.8} & 70.2 \\
			\hline
			21.2& 161.4 & \textBF{1423.0} \\
			\bottomrule
		\end{tabular}
		\caption{\method-PCA}
	\end{subtable}
	\begin{subtable}{.3\linewidth}
		\centering
		\begin{tabular}{r|r|r}
			\toprule
			$\mR_{ks}$ & $\mR_u$ & $\mN_{test}$ \\
			\midrule
			\textBF{194.6} & \multirow{2}{*}{45.4} & \multirow{2}{*}{58.8}\\  \cline{1-1}
			20.2 & & \\
			\hline
			0.0 & \textBF{198.0} & 74.6 \\
			\hline
			20.4 & 150.8 & \textBF{1459.2}\\
			\bottomrule
		\end{tabular}
		\caption{\method-ICA}
	\end{subtable}
\end{table}
\begin{table}[!h]
	\centering
	\caption{Sub-level classification: Breakdown confusion tables to compute \rarerate in \nyt.}
	\label{table:rare_rate_nyt}
	\begin{subtable}{.4\linewidth}
		\begin{tabular}{l|r|r|r}
			\toprule
			& $\mR_{ks}$ & $\mR_u$ & $\mN_{test}$ \\
			\midrule
			$\hat \mR_{ks}$ & \textBF{0.2} & \multirow{2}{*}{1.0} & \multirow{2}{*}{0.0} \\
			\cline{1-2}
			$\hat \mR_{k's}(k' \neq k)$ & 0 & & \\
			\hline
			$\mR_u$ & 291.0 & \textBF{330.4} & 2112.4 \\
			\hline
			$\mN$ & 0.2 & 340.0 & \textBF{0.6}\\
			\bottomrule
		\end{tabular}
		\caption{\senc}
	\end{subtable}
	\begin{subtable}{.24\linewidth}
		\begin{tabular}{r|r|r}
			\toprule
			$\mR_{ks}$ & $\mR_u$ & $\mN_{test}$ \\
			\midrule
			\textBF{3.0} & \multirow{2}{*}{420.8} & \multirow{2}{*}{1246.6} \\ \cline{1-1}
			175.6 & & \\
			\hline
			47.0 & \textBF{77.2} & 337.6 \\
			\hline
			65.8 & 163.4 & \textBF{528.8} \\
			\bottomrule
		\end{tabular}
		\caption{\lac}
	\end{subtable}
	\begin{subtable}{.24\linewidth}
		\begin{tabular}{r|r|r}
			\toprule
			$\mR_{ks}$ & $\mR_u$ & $\mN_{test}$ \\
			\midrule
			\textBF{158.4} & \multirow{2}{*}{132.2} & \multirow{2}{*}{175.2}\\\cline{1-1}
			12.6 & & \\
			\hline
			64.0 & \textBF{234.4} & 774.0 \\
			\hline
			56.4 & 294.8 & \textBF{1163.8} \\
			\bottomrule
		\end{tabular}
		\caption{\bl}
	\end{subtable}
	\begin{subtable}{.3\linewidth}
		\centering
		\begin{tabular}{r|r|r}
			\toprule
			$\mR_{ks}$ & $\mR_u$ & $\mN_{test}$ \\
			\midrule
			\textBF{158.4} & \multirow{2}{*}{135.8} & \multirow{2}{*}{175.2}\\\cline{1-1}
			12.6 & & \\
			\hline
			64.0 & \textBF{243.2} & 774.0 \\
			\hline
			56.4 & 282.4 & \textBF{1163.8} \\
			\bottomrule
		\end{tabular}
		\caption{\bl-r}
	\end{subtable}
	\begin{subtable}{.3\linewidth}
		\centering
		\begin{tabular}{r|r|r}
			\toprule
			$\mR_{ks}$ & $\mR_u$ & $\mN_{test}$ \\
			\midrule
			\textBF{176.2} & \multirow{2}{*}{251.0} & \multirow{2}{*}{69.0}\\ \cline{1-1}
			16.4 & & \\
			\hline
			56.2 & \textBF{260.2} & 77.4 \\
			\hline
			42.6 & 150.2 & \textBF{1966.6} \\
			\bottomrule
		\end{tabular}
		\caption{\method-1K}
	\end{subtable}
	\begin{subtable}{.3\linewidth}
		\centering
		\begin{tabular}{r|r|r}
			\toprule
			$\mR_{ks}$ & $\mR_u$ & $\mN_{test}$ \\
			\midrule
			\textBF{221.4} & \multirow{2}{*}{57.2} & \multirow{2}{*}{85.4} \\ \cline{1-1}
			44.2 & &  \\
			\hline
			0.0 & \textBF{400.0} & 129.0 \\
			\hline
			25.8 & 204.2 & \textBF{1898.6} \\
			\bottomrule
		\end{tabular}
		\caption{\method-PCA}
	\end{subtable}
	\begin{subtable}{.3\linewidth}
		\centering
		\begin{tabular}{r|r|r}
			\toprule
			$\mR_{ks}$ & $\mR_u$ & $\mN_{test}$ \\
			\midrule
			\textBF{213.2} & \multirow{2}{*}{62.4} & \multirow{2}{*}{296.2}\\ \cline{1-1}
			52.6 & & \\
			\hline
			0.0 & \textBF{393.6} & 0.4\\
			\hline
			25.6 & 205.4 & \textBF{1816.4}\\
			\bottomrule
		\end{tabular}
		\caption{\method-ICA}
	\end{subtable}
\end{table}

\subsection{Interpretability: Wordclouds for \method in all three datasets \riskd, \risks and \nyt.}
\label{sup:word_cloud}
Fig.~\ref{fig:word_cloud_nyt}, \ref{fig:word_cloud_risk} and \ref{fig:word_cloud_sent} show complete word clouds for both general and specialized classifiers for all three datasets considered in our paper.
\begin{figure}[!h] 
	\captionsetup[subfigure]{justification=centering}
	\begin{subfigure}[b]{0.3\linewidth}
		\centering
		\includegraphics[width=0.95\linewidth]{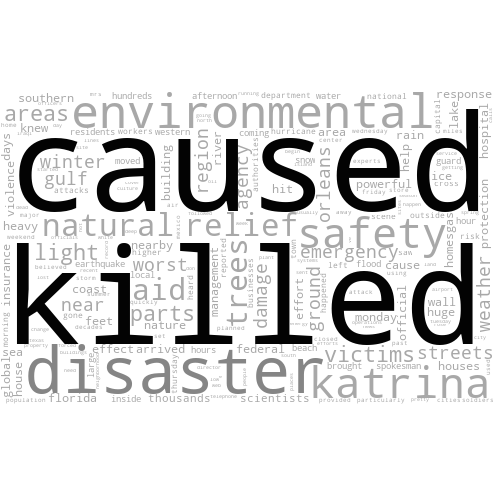} 
		\caption{General disaster} 
	\end{subfigure}
	\rulesep
	\begin{subfigure}[b]{0.3\linewidth}
		\centering
		\includegraphics[width=0.95\linewidth]{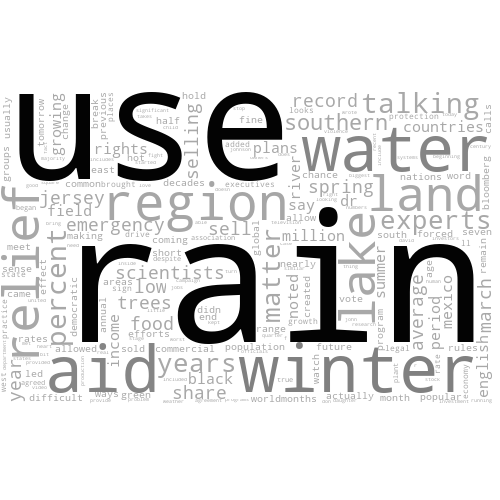} 
		\caption{Drought} 
	\end{subfigure}
	\begin{subfigure}[b]{0.3\linewidth}
		\centering
		\includegraphics[width=0.95\linewidth]{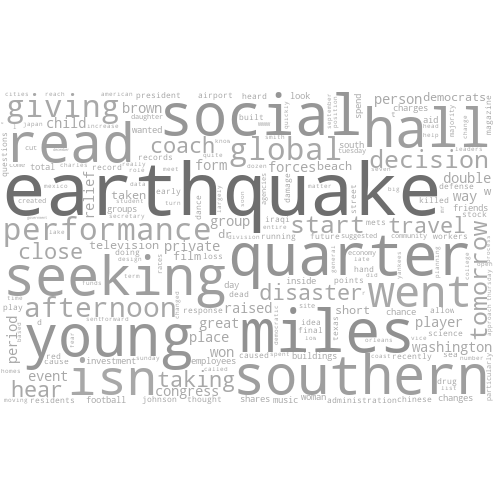} 
		\caption{Earthquakes} 
	\end{subfigure} \\ 
	\begin{subfigure}[b]{0.3\linewidth}
		\centering
		\includegraphics[width=0.95\linewidth]{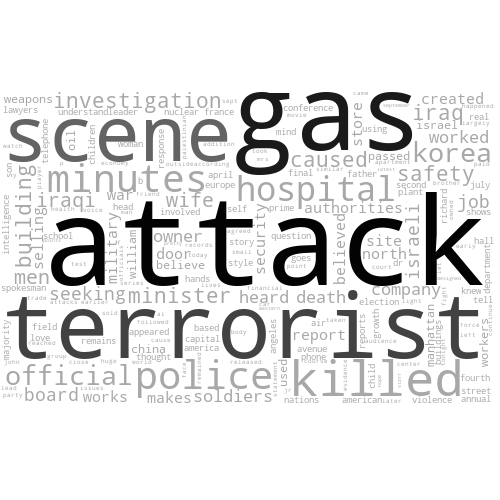} 
		\caption{Explosions} 
	\end{subfigure}
	\begin{subfigure}[b]{0.3\linewidth}
		\centering
		\includegraphics[width=0.95\linewidth]{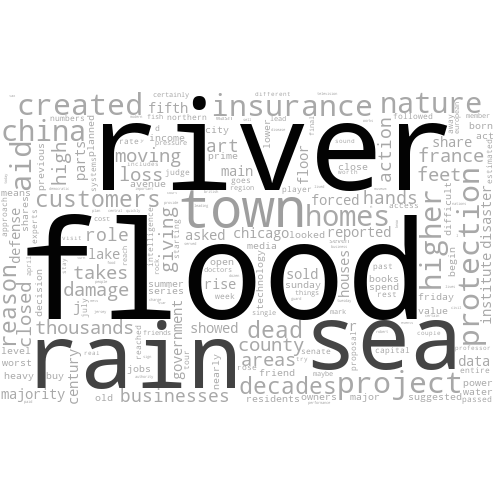} 
		\caption{Floods} 
	\end{subfigure}
	\begin{subfigure}[b]{0.3\linewidth}
		\centering
		\includegraphics[width=0.95\linewidth]{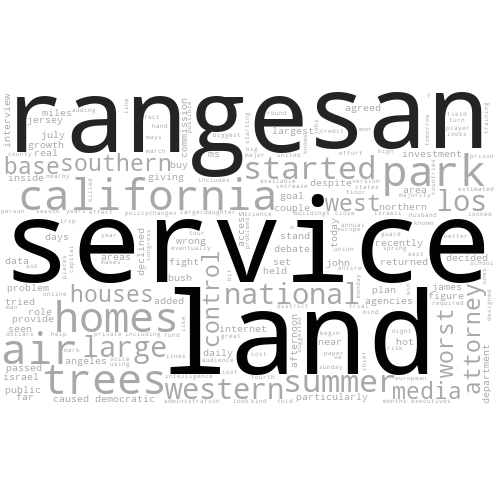} 
		\caption{Forest Fire} 
	\end{subfigure} \\
	\begin{subfigure}[b]{0.3\linewidth}
		\centering
		\includegraphics[width=0.95\linewidth]{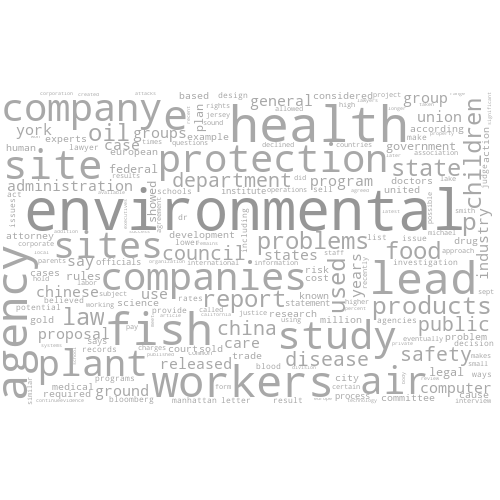} 
		\caption{Toxic Substance} 
	\end{subfigure}
	\begin{subfigure}[b]{0.3\linewidth}
		\centering
		\includegraphics[width=0.95\linewidth]{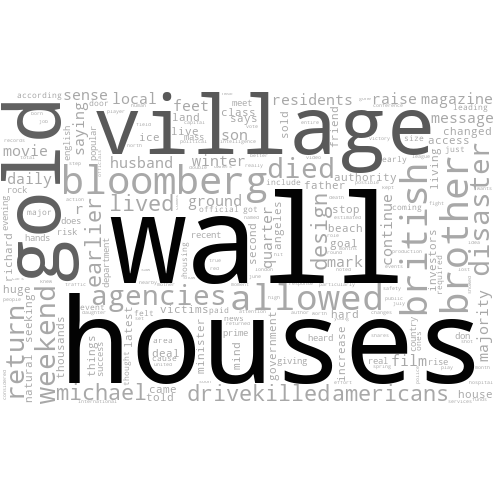} 
		\caption{Landslides} 
	\end{subfigure}
	\begin{subfigure}[b]{0.3\linewidth}
		\centering
		\includegraphics[width=0.95\linewidth]{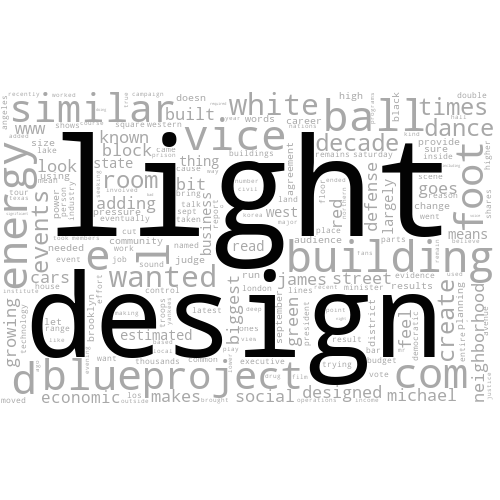} 
		\caption{Lightning} 
	\end{subfigure} \\
	\begin{subfigure}[b]{0.3\linewidth}
		\centering
		\includegraphics[width=0.95\linewidth]{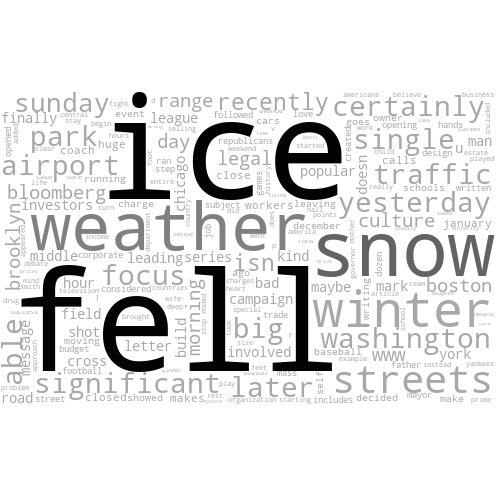} 
		\caption{Snowstorms} 
	\end{subfigure}
	\begin{subfigure}[b]{0.3\linewidth}
		\centering
		\includegraphics[width=0.95\linewidth]{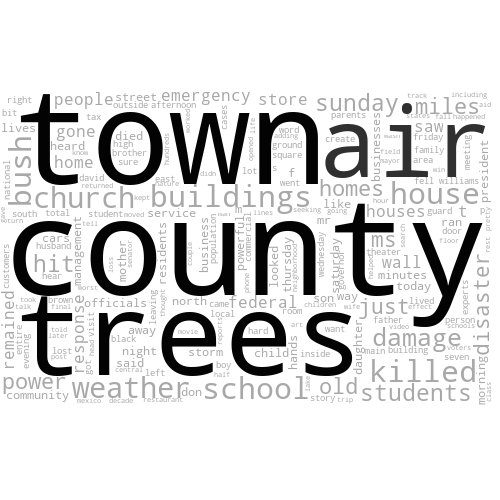} 
		\caption{Tornado}
	\end{subfigure}
	\begin{subfigure}[b]{0.3\linewidth}
		\centering
		\includegraphics[width=0.95\linewidth]{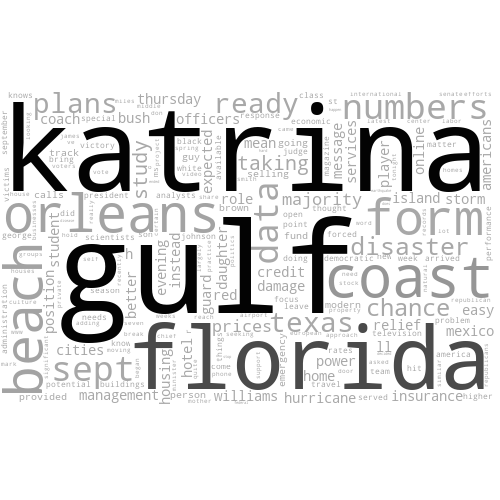} 
		\caption{Tropical Storms} 
	\end{subfigure} \\
	\begin{subfigure}[b]{0.3\linewidth}
		\centering
		\includegraphics[width=0.95\linewidth]{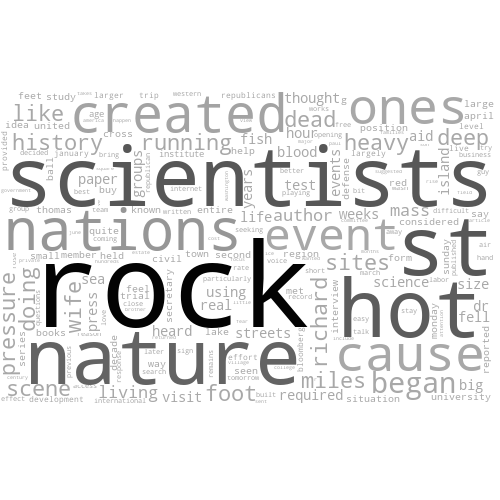} 
		\caption{Volcanoes} 
	\end{subfigure}
	\begin{subfigure}[b]{0.3\linewidth}
		\centering
		\includegraphics[width=0.95\linewidth]{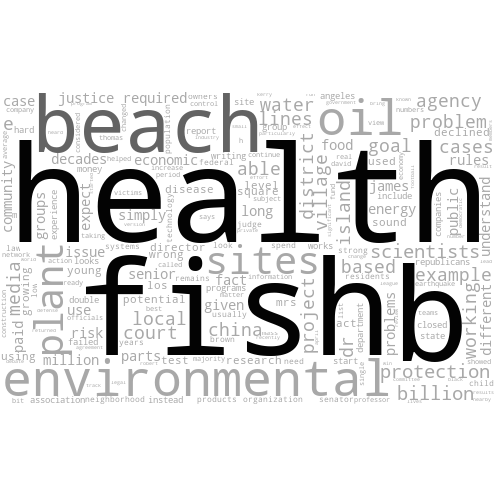} 
		\caption{Water Pollution} 
	\end{subfigure} 
	\caption{Interpretability: Word clouds representing learned weights by general and all specialized classifiers on \nyt dataset.}
	\label{fig:word_cloud_nyt} 
\end{figure}

\begin{figure}[!h] 
	\captionsetup[subfigure]{justification=centering}
	\begin{subfigure}[b]{0.3\linewidth}
		\centering
		\includegraphics[width=0.95\linewidth]{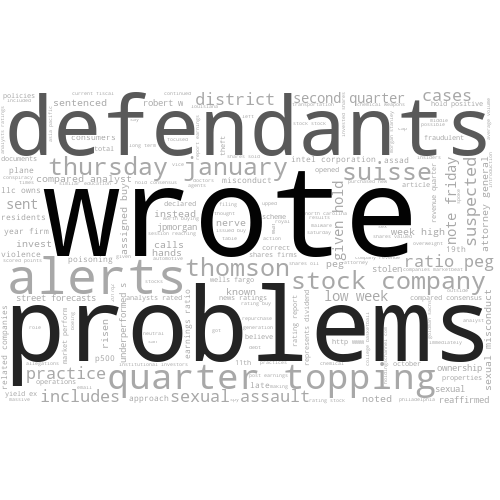} 
		\caption{General risk} 
	\end{subfigure}
	\rulesep
	\begin{subfigure}[b]{0.3\linewidth}
		\centering
		\includegraphics[width=0.95\linewidth]{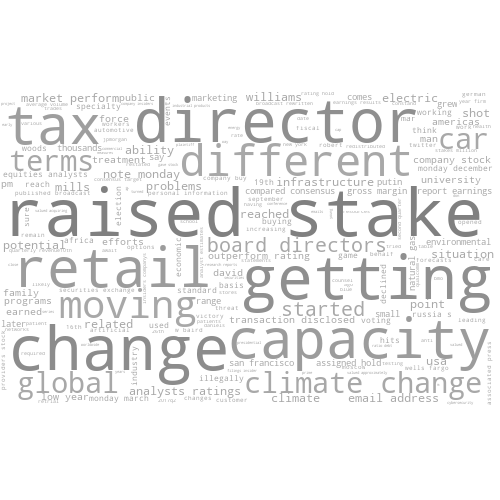} 
		\caption{Climate change} 
	\end{subfigure}
	\begin{subfigure}[b]{0.3\linewidth}
		\centering
		\includegraphics[width=0.95\linewidth]{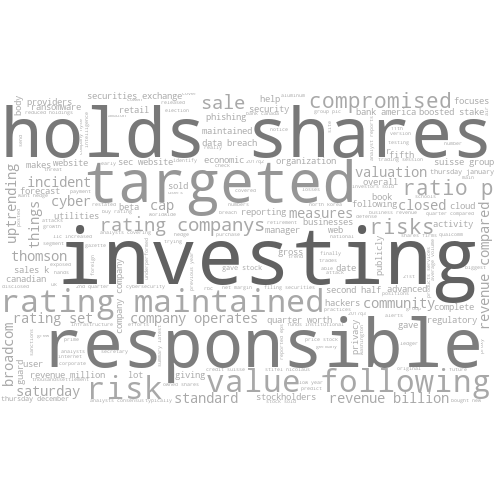} 
		\caption{Cyber-attack} 
	\end{subfigure} \\
	\begin{subfigure}[b]{0.3\linewidth}
		\centering
		\includegraphics[width=0.95\linewidth]{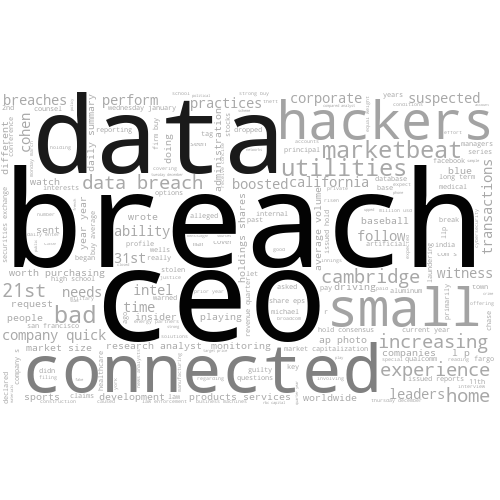} 
		\caption{Data leaks} 
	\end{subfigure} 
	\begin{subfigure}[b]{0.3\linewidth}
		\centering
		\includegraphics[width=0.95\linewidth]{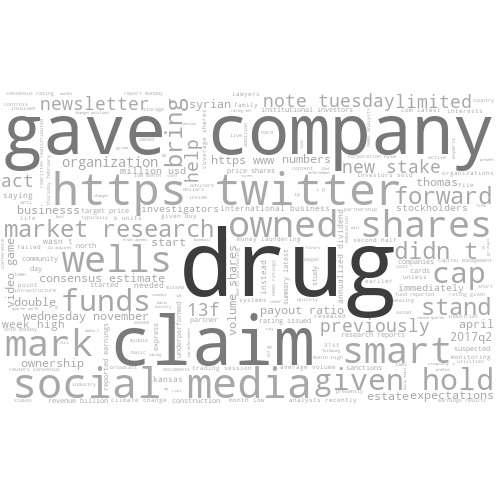} 
		\caption{Drug abuse} 
	\end{subfigure}
	\begin{subfigure}[b]{0.3\linewidth}
		\centering
		\includegraphics[width=0.95\linewidth]{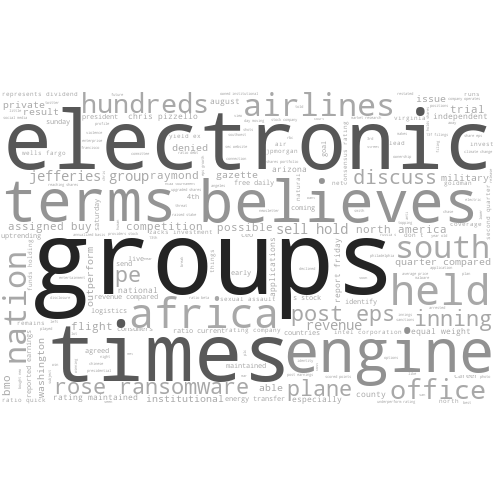} 
		\caption{Engine failure} 
	\end{subfigure} \\
	\begin{subfigure}[b]{0.3\linewidth}
		\centering
		\includegraphics[width=0.95\linewidth]{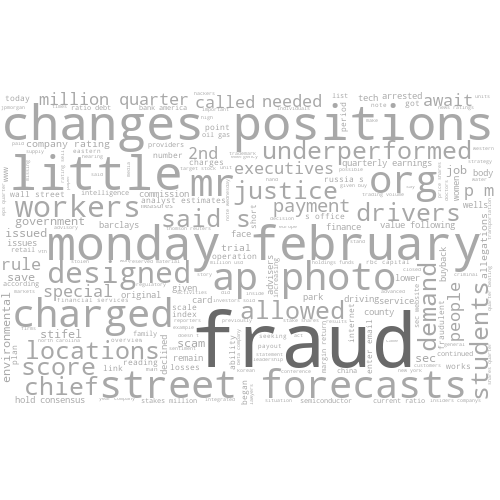} 
		\caption{Financial fraud} 
	\end{subfigure}
	\begin{subfigure}[b]{0.3\linewidth}
		\centering
		\includegraphics[width=0.95\linewidth]{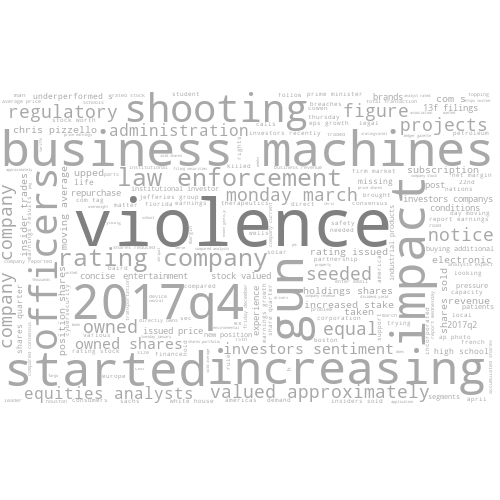} 
		\caption{Gun violation} 
	\end{subfigure}
	\begin{subfigure}[b]{0.3\linewidth}
		\centering
		\includegraphics[width=0.95\linewidth]{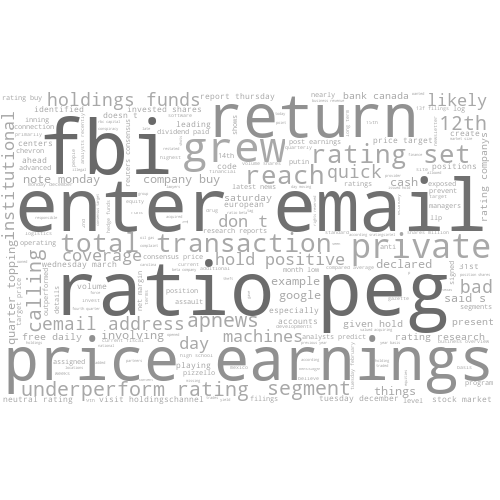} 
		\caption{Low stock rating} 
	\end{subfigure} \\
	\begin{subfigure}[b]{0.3\linewidth}
		\centering
		\includegraphics[width=0.95\linewidth]{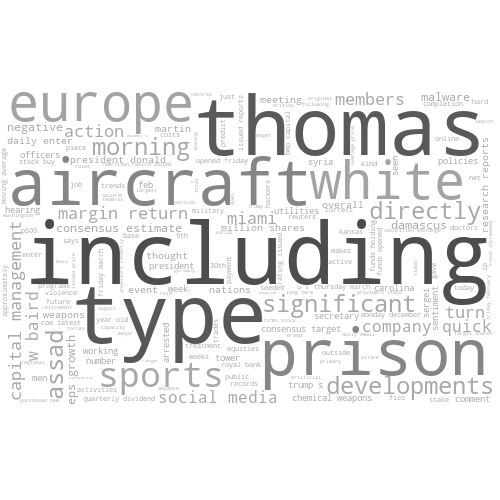} 
		\caption{Military attack} 
	\end{subfigure}
	\begin{subfigure}[b]{0.3\linewidth}
		\centering
		\includegraphics[width=0.95\linewidth]{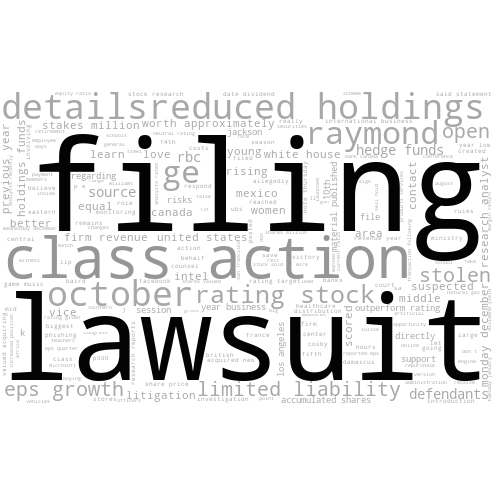} 
		\caption{Misleading statement}
	\end{subfigure} 
	\begin{subfigure}[b]{0.3\linewidth}
		\centering
		\includegraphics[width=0.95\linewidth]{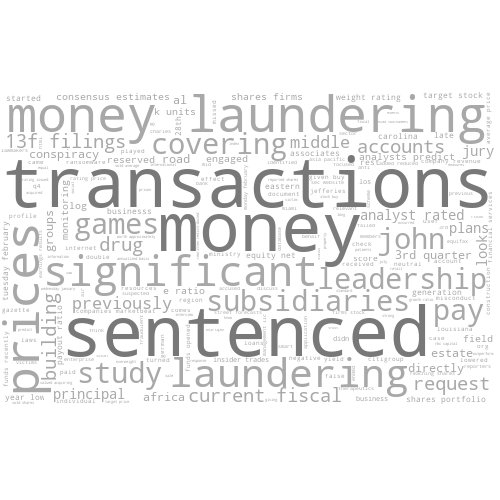} 
		\caption{Money laundering} 
	\end{subfigure}\\
	\begin{subfigure}[b]{0.3\linewidth}
		\centering
		\includegraphics[width=0.95\linewidth]{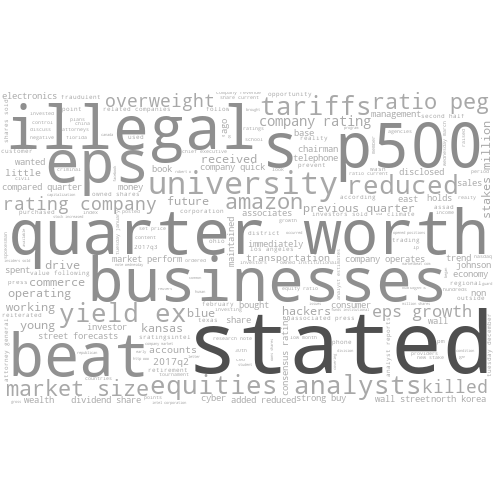} 
		\caption{Negative earning} 
	\end{subfigure} 
	\begin{subfigure}[b]{0.3\linewidth}
		\centering
		\includegraphics[width=0.95\linewidth]{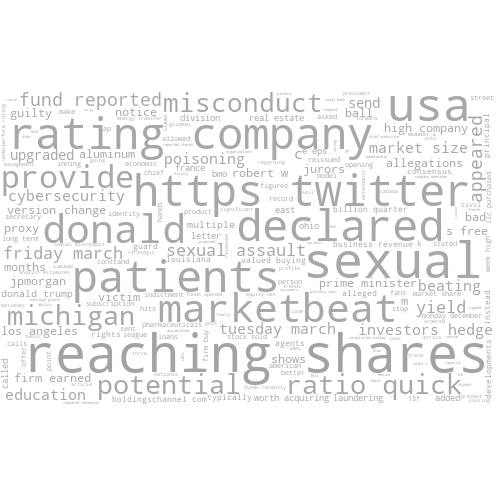} 
		\caption{Sexual assault} 
	\end{subfigure}
	\begin{subfigure}[b]{0.3\linewidth}
		\centering
		\includegraphics[width=0.95\linewidth]{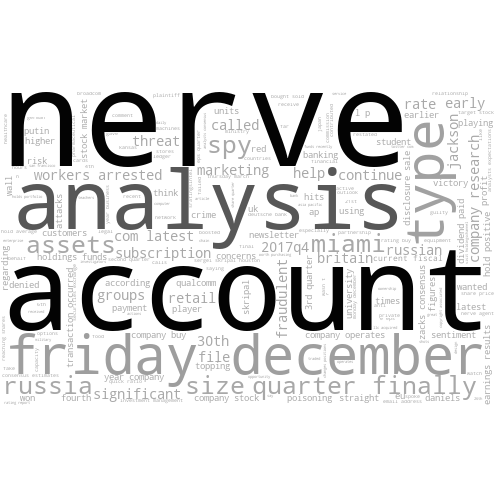} 
		\caption{Spying} 
	\end{subfigure}
\end{figure}
\begin{figure}[htb]\ContinuedFloat
	\begin{subfigure}[b]{0.3\linewidth}
		\centering
		\includegraphics[width=0.95\linewidth]{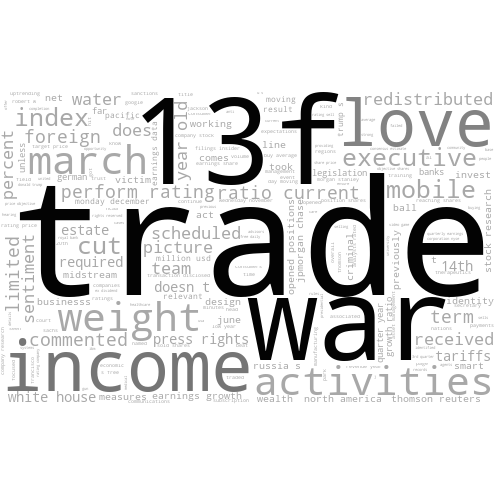} 
		\caption{Trade war} 
	\end{subfigure} 
	\caption{Interpretability: Word clouds representing learned weights by general and all specialized classifiers on \riskd dataset.}
	\label{fig:word_cloud_risk} 
\end{figure}

\begin{figure}[!h] 
	\captionsetup[subfigure]{justification=centering}
	\begin{subfigure}[b]{0.3\linewidth}
		\centering
		\includegraphics[width=0.95\linewidth]{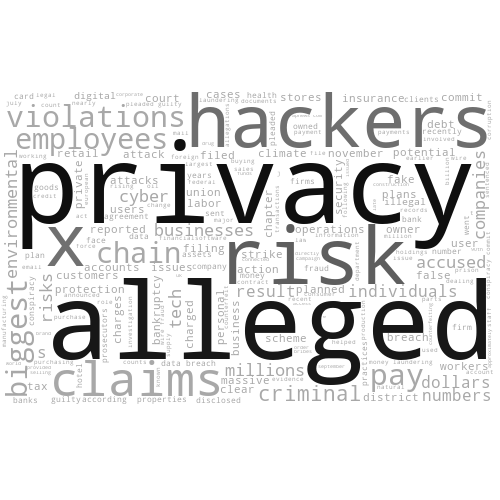} 
		\caption{General risk} 
	\end{subfigure}
	\rulesep
	\begin{subfigure}[b]{0.3\linewidth}
		\centering
		\includegraphics[width=0.95\linewidth]{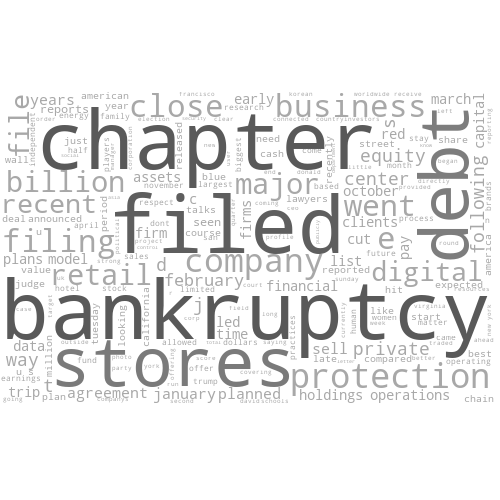} 
		\caption{Bankruptcy} 
	\end{subfigure} 
	\begin{subfigure}[b]{0.3\linewidth}
		\centering
		\includegraphics[width=0.95\linewidth]{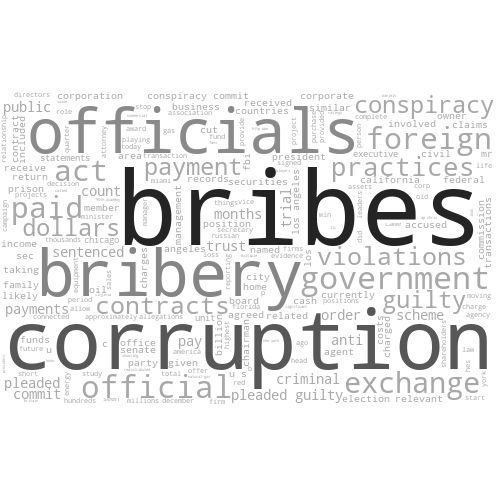} 
		\caption{Corruption} 
	\end{subfigure} \\
	\begin{subfigure}[b]{0.3\linewidth}
		\centering
		\includegraphics[width=0.95\linewidth]{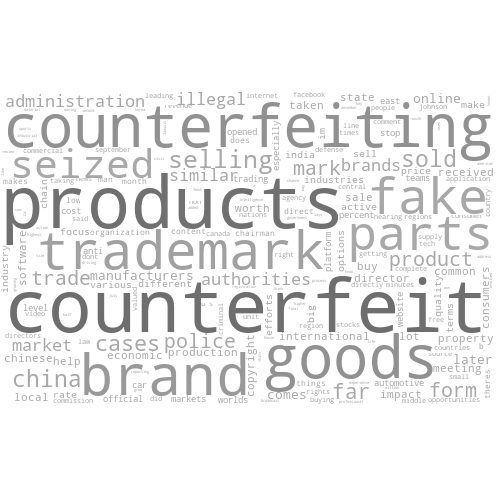} 
		\caption{Counterfeiting} 
	\end{subfigure}
	\begin{subfigure}[b]{0.3\linewidth}
		\centering
		\includegraphics[width=0.95\linewidth]{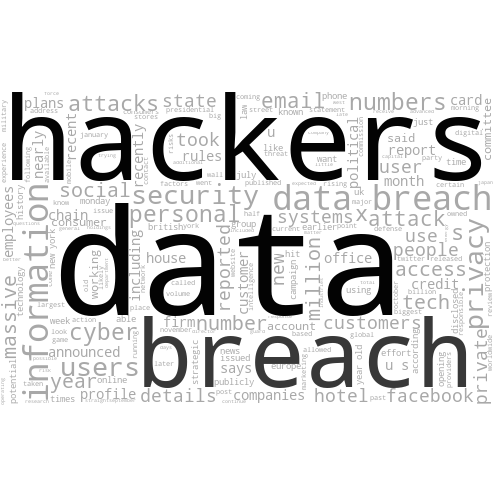} 
		\caption{Cyber-privacy} 
	\end{subfigure}
	\begin{subfigure}[b]{0.3\linewidth}
		\centering
		\includegraphics[width=0.95\linewidth]{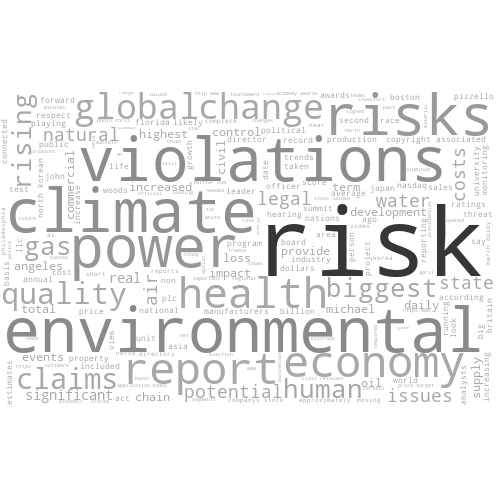} 
		\caption{Environment} 
	\end{subfigure} \\
	\begin{subfigure}[b]{0.3\linewidth}
		\centering
		\includegraphics[width=0.95\linewidth]{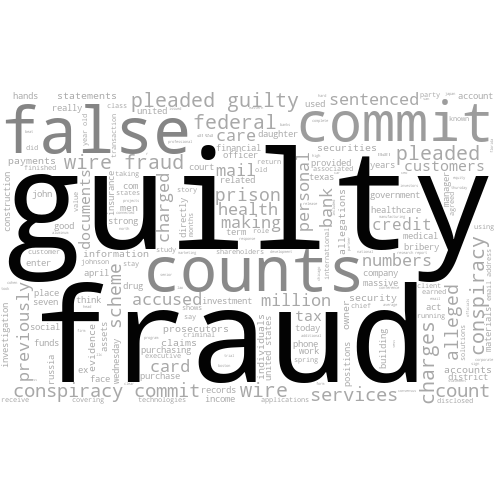} 
		\caption{Fraud False Claims} 
	\end{subfigure} 
	\begin{subfigure}[b]{0.3\linewidth}
		\centering
		\includegraphics[width=0.95\linewidth]{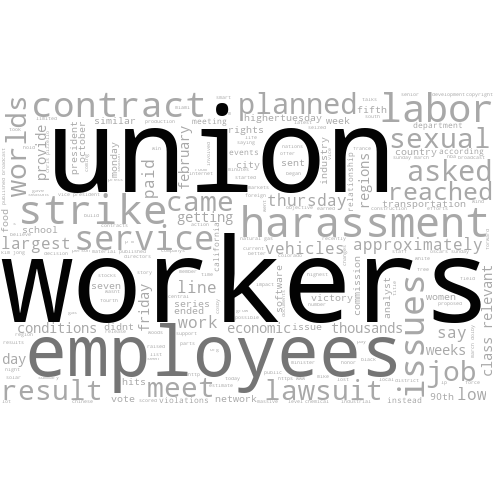} 
		\caption{Labor} 
	\end{subfigure} 
	\begin{subfigure}[b]{0.3\linewidth}
		\centering
		\includegraphics[width=0.95\linewidth]{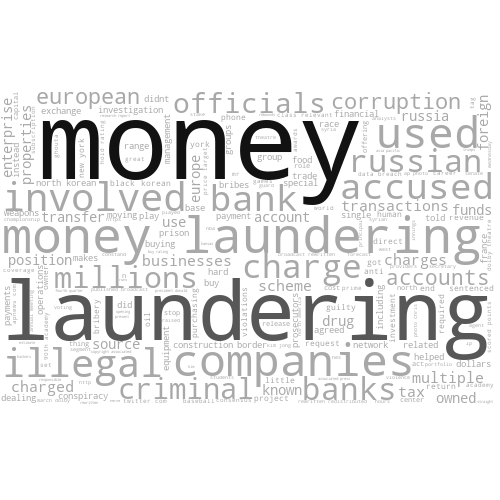} 
		\caption{Money laundering}
	\end{subfigure} 
	\caption{Interpretability: Word clouds representing learned weights by general and all specialized classifiers on \risks dataset.}
	\label{fig:word_cloud_sent} 
\end{figure}

\end{document}